\documentclass{article}

\usepackage[main, final]{neurips_2026}

\usepackage{natbib}

\usepackage[utf8]{inputenc} 
\usepackage[T1]{fontenc}    
\usepackage{hyperref}       
\usepackage{url}            
\usepackage{booktabs}       
\usepackage{amsfonts}       
\usepackage{nicefrac}       
\usepackage{microtype}      
\usepackage{xcolor}         

\usepackage{amsthm}
\usepackage{enumitem}
\usepackage{wrapfig}
\usepackage{algorithm}
\usepackage{algorithmic}
\usepackage{amsmath, amssymb}
\usepackage{thm-restate}

\newtheorem{definition}{Definition}
\newtheorem{theorem}{Theorem}
\newtheorem{lemma}{Lemma}
\newtheorem{proposition}{Proposition}

\newtheorem{corollary}{Corollary}

\newtheorem{assumption}{Assumption}

\DeclareMathOperator*{\E}{\mathbb{E}}

\newcommand{\cA}{\mathcal{A}}
\newcommand{\cD}{\mathcal{D}}
\newcommand{\cZ}{\mathcal{Z}}
\newcommand{\cH}{\mathcal{H}}

\usepackage{amsmath,amsfonts,bm}

\def\eqref#1{equation~\ref{#1}}

\def\1{\bm{1}}

\DeclareMathAlphabet{\mathsfit}{\encodingdefault}{\sfdefault}{m}{sl}
\SetMathAlphabet{\mathsfit}{bold}{\encodingdefault}{\sfdefault}{bx}{n}

\usepackage{hyperref}
\usepackage{url}
\usepackage{amsthm}
\usepackage{caption}
\usepackage{placeins}
\usepackage{soul}
\usepackage{graphicx}

\usepackage[normalem]{ulem}

\title{Certification-Enhanced Generalization Bounds}

\author{
Leo Elmecker-Plakolm \\
Department of Computing\\
Imperial College London\\
London, SW7 2AZ, UK\\
\texttt{le24@ic.ac.uk} \\
\And
Matthew Wicker \\
Department of Computing\\
Imperial College London\\
London, SW7 2AZ, UK\\
\texttt{mwicker@imperial.ac.uk} \\
}

\begin{document}

\maketitle

\begin{abstract}
  We investigate the use of formal methods to provide tight and sound generalization bounds for learning algorithms. By casting the traditional notion of algorithmic stability as a specification to be verified, we demonstrate that recent advances in reachability analysis can yield provable bounds on the generalization of a given model and algorithm on a sample dataset. As sample-specific algorithmic stability is insufficient to bound  the usual distributional notion of generalization, we develop a novel concentration inequality to connect the sample-specific results of formal certification algorithms to the required distributional analysis for bounding the expected generalization gap. The resulting framework enables the analysis of prior generalization bounds to extend far beyond their original restrictive assumptions. Our approach computes sound bounds on the expected generalization gap in a constant number of algorithm runs without making any analytical assumptions on the algorithm; to achieve non-vacuous bounds we only require that the certified reachable parameter set is bounded --- a condition that we do not assume but formally verify. In practice, we demonstrate that our framework provides formal generalization guarantees that are orders of magnitude tighter than alternative sound computational approaches at scales ranging from toy datasets to fine-tuning classification heads on top of modern large language models. While we implement certification-enhanced versions of several well-known stability results, future extensions of our approach will enable tighter bounds and enhanced practical adoption across the spectrum of modern generalization bounds.
\end{abstract}

\section{Introduction}\label{sec:Introduction}

\looseness=-1
A central aim of statistical learning theory is to characterize the gap
between a model's performance on its training data and its performance at
deployment. To this end, a substantial body of work has sought to
upper-bound this \textit{generalization gap} through frameworks such as
PAC-Bayes analysis \citep{germain2013pac, mcallester1998some}, Rademacher complexity \citep{Bartlett2003RademacherAG}, and algorithmic
stability \citep{bousquet2002stability, hardt2016train}. Tight bounds on
the generalization gap serve two complementary purposes: they guide the
design of algorithms with improved out-of-sample performance, and they
provide practitioners with formal assurances that empirically observed
performance will persist after deployment. Yet, as modern machine learning
systems increasingly diverge from the analytical regimes for which
classical bounds were designed, the assumptions underlying these bounds
strain under empirical reality. Bounds that remain faithful to their
stated conditions often require restrictive, fine-grained assumptions
that limit their practical applicability, while efforts to recover
tightness in modern settings increasingly rely on numerical
approximations of critical quantities, thus eroding their practical assurances \citep{kuzborskij2018data,
fan2024high}.

\looseness=-1
In this work, we investigate whether recent advances in formal methods
can deliver generalization bounds that are simultaneously sound and
tight. Our central idea is to recast notions of algorithmic stability as
formal specifications amenable to verification by recently developed algorithms \mbox{\citep{sosnin2024certified}}. We focus on three
classical variants: \textit{uniform stability} \citep{bousquet2002stability, hardt2016train}, which requires that the
loss of the learned hypothesis on an arbitrary instance change by at most a bounded amount under any
single-sample substitution in the training set; \textit{data-dependent
stability} \citep{kuzborskij2018data}, which modifies this definition to single-sample substitutions drawn from the underlying data distribution; and \textit{on-average stability} \citep{shalevshwartz2010onaveragestability}, which requires a bound on the expected loss of instances from the data distribution rather than the worst-case instance. For each of
these notions, we observe that analytically bounding the corresponding
stability constant---even for small multi-layer perceptrons---yields
vacuous bounds, as
worst-case reasoning over arbitrary substitutions ignores the
structure of the specific algorithm and dataset at hand. While approximations can make bounds non-vacuous, they sacrifice soundness of the analysis, making the reliability of estimated bounds questionable.

\looseness=-1
\begin{wrapfigure}{r}{0.6\textwidth}
    \centering
    \includegraphics[width=\linewidth]{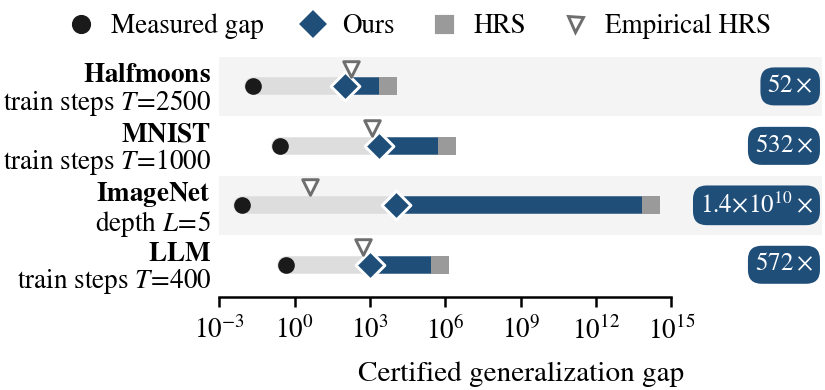}
    \caption{In various settings our bounds are always-sound, tighter alternatives to existing analytical bounds.}
    \label{fig:halfmoons_acc}
\end{wrapfigure}

Soundness is, of course, the \textit{raison d'\^etre} of formal methods, but
existing notions of stability are not directly amenable to off-the-shelf
verification tools. We therefore introduce \textit{sample-specific}
analogues of the above notions: uniform, data-dependent uniform, and on-average stability. A sample-specific stability
condition asks whether a given algorithm (i.e.,
architecture and the optimization procedure) satisfies the stability
criterion on a particular sampled dataset, rather than uniformly over all
possible datasets or a given distribution. We begin by showing that these sample-specific conditions can be
formally proven by recent reachability and certification techniques for neural
networks. In particular, we instantiate our framework using Abstract
Gradient Training (AGT) \citep{sosnin2025abstract}, which produces provably
sound bounds on the parameter-space trajectory under worst-case
single-sample perturbations. A salient feature of this formulation is
its frugality with respect to assumptions: we make no formal assumptions on the learning algorithm or its loss function. We show, however, that our algorithm requires \textit{effective boundedness} (Assumption~\ref{ass:eff_boundedness}) in order for our bounds to be finite.

Sample-specific stability is verifiable, but on its own it does not bound
generalization in the classical sense, which requires guarantees that
hold over the unknown distribution from which the dataset is drawn. To
close this gap, we develop a novel concentration inequality to connect
our sample-specific certificates to traditional stability-based
generalization bounds. The construction draws inspiration from tools
developed in differential privacy \citep{nissim2007smoothsensitivity}.
By defining sample-specific \textit{$k$-stability} as the stability of
the algorithm to $k$ simultaneous substitutions, we obtain a computable,
sound proxy for the smoothness of the verification result as a function
of the input dataset. This in turn yields a high-probability guarantee
that the verified bound holds for any dataset $S$ drawn from the same
distribution as the observed sample.

We empirically validate our framework across a diverse suite of settings, ranging from controlled toy problems to image classification benchmarks and the fine-tuning of modern large language model classification heads. Our bounds are orders of magnitude tighter than existing analytically sound alternatives and are applicable to substantially more complex optimization problems. In practice, our bounds are competitive with numerical approximations while retaining provable validity with high probability, rendering them a more trustworthy signal for model selection than bounds whose tightness reflects the noise of numerical approximation.
\section{Related Works}\label{sec:related_works}

\paragraph{Algorithmic Stability Bounds.}
Our methodology builds on the notion of algorithmic stability, which shifts the focus from the complexity of the hypothesis space to the sensitivity of the learning algorithm itself. The foundational concept of uniform stability, introduced by \citet{bousquet2002stability}, bounds the maximum change in loss when a single training point, sampled from anywhere in the instance space, is replaced. This was later applied to Stochastic Gradient Descent (SGD) by \citet{hardt2016train}, demonstrating that stability can be bounded using properties of the optimization problem.

To tighten bounds and apply them to higher-dimensional spaces, research has evolved from uniform (worst-case) measures toward data-dependent stability. \citet{shalevshwartz2010onaveragestability} proposed on-average stability, which relaxes the worst-case requirement on test instances, while \citet{kuzborskij2018data} introduced bounds that relax the worst-case dataset requirement, thereby producing bounds that only hold for a given data distribution. In parallel, some works have established local bounds, more similar to this work. For instance, \citet{deng2021bettergeneralizationboundslocally} and \citet{fan2024high} acknowledge that a model's sensitivity is not uniform across the data manifold and can be input sample specific, but are unable to provide provably sound upper bounds on their notion of sensitivity.

A related branch of work has derived bounds under a similar premise. \citet{soloff2024bagging} and their extension, \citet{liang2025assumptionfreestabilityrankingproblems}, establish assumption-free stability guarantees via a bagging algorithm. While the goal of our work and \citet{soloff2024bagging} is similar, \citet{soloff2024bagging} can only provide guarantees for specifically designed bagging algorithms as opposed to our less intrusive framework, which certifies any learning algorithm (in practice our implementation only applies to gradient-based algorithms). Hence, we consider these works to be complementary rather than strictly competing.

\paragraph{Further Approaches to Generalization Bounds.}

Seminal results on the generalization gap derive bounds that hold
uniformly over a fixed hypothesis class, controlling the worst-case gap
through capacity-based notions of complexity such as the VC dimension
\citep{vapnikVCdim1971} or Rademacher complexity
\citep{Bartlett2003RademacherAG}. While foundational, such bounds are
typically vacuous when applied to modern overparameterized models,
offering little actionable signal for design or deployment
\citep{dziugaite2017computingnonvacuousgeneralizationbounds}. Subsequent work has therefore sought
bounds that depend more directly on the algorithm and data at hand.
Information-theoretic frameworks quantify generalization through the
\textit{information leak} between the training data and the learned
hypothesis: \citet{russo2016controllingbiasinadaptivedata} and
\citet{xu2017informationtheoreticanalysisgeneralizationcapability} bound
the generalization error in terms of the mutual information between the
dataset and the algorithm's output, an analysis subsequently extended to
noisy learning regimes \citep{pensia2018generalizationerrorboundsnoisy}
and sharpened to a conditional mutual-information form by
\citet{steinke2020reasoninggeneralizationconditionalmutual}. The closely
related PAC-Bayesian framework \citep{mcallester1998some,
germain2013pac} controls the gap through the KL-divergence between a
data-dependent posterior and a fixed prior over hypotheses. While we focus on stability-based bounds in this paper, we hope that a similar approach to the one we take here will be applied to these bounds in the future. 

\looseness=-1
\paragraph{Formal Methods and Abstract Gradient Training.}
Formal methods in machine learning and artificial intelligence can be used to prove certain safety properties of a given neural network \citep{katz2017reluplex, wicker2018feature}. In its early stages, formal verification mainly focused on post-hoc verification of trained neural networks \citep{gehr2018ai2}. Later, these approaches were applied to networks with intervals over their weights \citep{wicker2020probabilistic, wicker2024probabilistic} or optimize weight-space intervals \citep{elmecker2025provably}. The application of these bounds to gradients \citep{wicker2022robust} enabled them to be applied to training algorithms. \citet{sosnin2025abstract} introduced \textit{Abstract Gradient Training (AGT)}, which focuses on verification as part of the training process. AGT can be used to bound the worst-case impact of an adversary during training in settings like differential privacy \citep{wicker2024certification, ghitu2026certification} or data poisoning \citep{sosnin2024certified, sosnin2025abstract}. To propagate valid gradient and sample bounds through a network during training, AGT mostly relies on interval bound propagation (IBP) \citep{gowal2019effectivenessintervalboundpropagation} or CROWN \citep{zhang2018efficientneuralnetworkrobustness}. In this work, we leverage AGT to find worst-case bounds for a neural network under a specific learning algorithm considering a given substitution budget.

\section{Preliminaries}
Let $\cZ$ denote the instance space and $\cD$ a distribution over $\cZ$.
A training set is defined as $S=(z_1,\ldots,z_n)\sim\cD^n$, which a
learning algorithm $\cA$ maps to a hypothesis $\cA_S\in\cH$; throughout,
we assume that any algorithmic randomness is realized and fixed
independently of $S$. We denote a loss function
$\ell:\cH\times\cZ\to\mathbb R$. Existing bounds commonly assume that
this loss is bounded by a known constant or is Lipschitz continuous with
a known global constant. Here, we place no convexity, smoothness, or
global Lipschitz assumptions on the learning algorithm, as formal
verification takes the algorithm \textit{as is}. In
Assumption~\ref{ass:eff_boundedness}, we state the condition under which
the required loss range can itself be certified.

The \emph{generalization gap} is the difference between the population
risk, $R(\cA_S)=\E_{z\sim\cD}[\ell(\cA_S,z)]$, and the empirical risk,
$\hat R_S(\cA_S)=\frac{1}{n}\sum_{i=1}^n\ell(\cA_S,z_i)$. We concern
ourselves with upper-bounding the expectation of this gap over training
samples. Throughout, $S^{(i)}=(z_1,\ldots,z_i',\ldots,z_n)$ denotes
replacement of $z_i$ by $z_i'\in\cZ$. The Hamming
distance $d(S,S^\star)$ counts the coordinates at which two datasets of
equal size differ, and $[n]:=\{1,\ldots,n\}$. We assume that the
expectations used below exist and that all functions and suprema are
measurable.

\subsection{Standard Generalization Bounds from Stability}
\label{sec:standard_gen_bounds_from_stability}
Algorithmic stability bounds control the change in loss incurred by
applying a learning algorithm to two datasets with a single sample
change. We begin with the most restrictive notion, which provides a
global guarantee over both datasets and instances, before relaxing these
requirements to account for the data distribution.

\begin{definition}[Algorithmic Stability]
\label{def:algorithmic_stability}
\label{def:uniform_stability}
\label{def:on_average_stability}
\label{def:data_dependent_stability}
The following quantities describe the stability of $\cA$:
\begin{itemize}[leftmargin=*,itemsep=0.5pt,topsep=0.5pt,parsep=0pt]
    \item \textbf{Uniform stability} bounds the worst-case loss change over
    all training datasets, replacements, and test instances:
    $\bar\epsilon:=\sup_{S,i,z_i',z}\bigl|\ell(\cA_S,z)-\ell(\cA_{S^{(i)}},z)\bigr|$.
    This gives a distribution-free guarantee \citep{hardt2016train}.
    \item \textbf{Data-dependent uniform stability} averages over the
    training distribution while retaining the worst-case replacement and
    test instance:
    $\bar\epsilon(\cD):=\sup_{i,z_i',z}\E_S\bigl|\ell(\cA_S,z)-\ell(\cA_{S^{(i)}},z)\bigr|$.
    The bound is therefore weighted by the training datasets likely
    under $\cD$.
    \item \textbf{On-average stability} also averages over a fresh
    instance:
    $\hat\epsilon(\cD):=\max_{i\in[n]}\E_{S,z_i'}[\ell(\cA_S,z_i')-\ell(\cA_{S^{(i)}},z_i')]$.
\end{itemize}
\end{definition}

Computing uniform stability requires global bounds over neighboring
training datasets and over the test instance. These optimization problems
are usually intractable for deep neural networks, and analytical bounds
have consequently relied on properties of the learning algorithm and
loss \citep{hardt2016train,kuzborskij2018data}. Averaging over $\cD$
relaxes these requirements, but still requires reasoning over training
datasets that have not been observed. We next introduce quantities that
can instead be certified from a particular sample.

\section{Data-Specific Generalization Bounds}
\label{sec:data-specific-bounds}
To make notions of stability amenable to certification, we introduce
sensitivity measures that depend on the observed dataset $S$. They
naturally extend the population notions in
Definition~\ref{def:algorithmic_stability}: the argument $\cD$ describes
averaging over training samples, whereas the argument $S$ holds the
training sample fixed.

\begin{definition}[Dataset-Specific Stability]
\label{def:data_specific_stability}
For an observed sample $S=(z_1,\ldots,z_n)$, define:
\begin{itemize}[leftmargin=*,itemsep=0.5pt,topsep=0.5pt,parsep=0pt]
    \item \textbf{Uniform dataset-specific stability} fixes $S$ in the
    uniform notion while retaining the worst-case replacement and
    instance:
    $\bar\epsilon(S):=\sup_{i,z_i',z}\bigl|\ell(\cA_S,z)-\ell(\cA_{S^{(i)}},z)\bigr|$.
    \item \textbf{On-average dataset-specific stability} retains the
    average over test instances from the on-average notion, but bounds
    every replacement of the observed sample:
    $\hat\epsilon(S):=\max_{i\in[n]}\E_{z\sim\cD}\sup_{z_i'\in\cZ}[\ell(\cA_S,z)-\ell(\cA_{S^{(i)}},z)]$.
\end{itemize}
Each quantity is non-negative because replacement by the unchanged
example is allowed.
\end{definition}

Unlike the population notions, dataset-specific stability depends on the
sample at hand. Consequently, its value at a single $S$ does not directly
bound the expected generalization gap. A naive but valid approach is to
lift it to a population quantity by taking an expectation over $S$:

\begin{proposition}[Moving from Data-Specific to Data-Dependent Stability]
\label{prop:jensen}
By Jensen's inequality,
$\hat\epsilon(\cD)\le\E_S[\hat\epsilon(S)]$ and
$\bar\epsilon(\cD)\le\E_S[\bar\epsilon(S)]$.
These imply the expected generalization bound $\E_S[R(\cA_S)-\hat R_S(\cA_S)]\le\E_S[\epsilon(S)]$ for either dataset-specific quantity $\epsilon$.
\end{proposition}

The proposition shows that the dataset-specific quantities are
conservative in expectation and can therefore be used to bound the
expected gap. In what follows, we write $\epsilon(S)$ for whichever of
these notions is used. Full proofs are deferred to
Appendix~\ref{app:proofs}.

\subsection{Smooth Data-Specific Bounds Yield Generalization Guarantees}
However, to use Proposition~\ref{prop:jensen} directly, one would need to
compute dataset-specific stability for many observed datasets. This is
at least data-inefficient and, in the worst case, infeasible. We therefore
bound its expectation by considering how stability changes with the
training sample. Inspired by smooth sensitivity from differential
privacy \citep{nissim2007smoothsensitivity}, we use the following smooth
upper bound.

\begin{definition}[$\beta$-Smooth Dataset-Specific Sensitivity]
\label{def:beta_smooth_stability_at_S}
For $\beta>0$, the $\beta$-smooth stability at $S$ is
$B_\beta^*(S):=\sup_{S^\star\in\cZ^n}\bigl(\epsilon(S^\star)e^{-\beta d(S,S^\star)}\bigr)$.
\end{definition}

\begin{proposition}[Properties of $B_\beta^*$]
\label{prop:properties_of_B}
The function $B_\beta^*(S)$ satisfies the following properties:
\begin{enumerate}[leftmargin=*,itemsep=0.5pt,topsep=0.5pt,parsep=0pt]
    \item \textbf{Upper Bound:} $B_\beta^*(S)\ge\epsilon(S)$ for all
    $S\in\cZ^n$.
    \item \textbf{Log-Lipschitz Continuity:} If $B_\beta^*$ is
    positive and finite,
    $|\ln B_\beta^*(S)-\ln B_\beta^*(S')|\le\beta d(S,S')$.
\end{enumerate}
\end{proposition}

Provided that one can upper-bound the smooth stability at $S$, the
subject of Section~\ref{sec:algorithm}, we can now obtain a generalization
guarantee from the observed dataset.

\begin{theorem}[Generalization from Smooth Stability]
\label{th:th_gen_bound_beta_smooth}
Let $B_\beta^*$ be finite and positive, and fix $\beta>0$ and $\delta\in(0,1)$
independently of $S_{\mathrm{obs}}\sim\cD^n$. With probability at least
$1-\delta$,
\begin{equation}
\E_S[R(\cA_S)-\hat R_S(\cA_S)]
\le B_\beta^*(S_{\mathrm{obs}})\cdot
\exp\!\left(\beta\sqrt{\frac{n}{2}\log\frac{1}{\delta}}
                 +\frac{n\beta^2}{8}\right).
\label{eq:gen_bound_beta_smooth}
\end{equation}
A sound upper bound on $B_\beta^*(S_{\mathrm{obs}})$ may be substituted
on the right-hand side.
\end{theorem}

We note that Theorem~\ref{th:th_gen_bound_beta_smooth} follows previous work regarding algorithmic-stability-based generalization bounds concerning the \textit{expected} gap \citep{hardt2016train, bousquet2002stability} as opposed to the \textit{realized} gap, which also is of concern in some adjacent research \citep{fan2024high}. Appendix~\ref{app:realized-gap} shows a sketch of how the global certificate can also yield a looser high-probability bound on the realized gap.

The asymptotic behavior of our bound may appear unusual because its
exponential prefactor grows with $n$ for fixed $\beta$. However, taking
$\beta=\frac{c}{\sqrt n}$ for a constant $c>0$ reduces the prefactor in
Eq.~\ref{eq:gen_bound_beta_smooth} to the constant
$K(c,\delta):=\exp\bigl(c\sqrt{\frac{1}{2}\log\frac{1}{\delta}}
+\frac{c^2}{8}\bigr)$.
The asymptotic question therefore reduces to the rate at which
$B_\beta^*(S_{\mathrm{obs}})$ vanishes with $n$. We first show that
linear growth of stability over neighboring datasets gives a vanishing
$O(\frac{1}{\sqrt n})$ bound, before considering the stronger rates
available under classical stability assumptions.

\begin{proposition}[Asymptotic Vanishing under Linear Neighborhood Stability]
\label{prop:asymptotic}
Suppose the learning algorithm $\cA$ satisfies, at the observed sample,
\begin{equation}
\sup_{S^\star:\,d(S_{\mathrm{obs}},S^\star)\le k}\epsilon(S^\star)
\le\frac{C_T(k+1)}{n}
\quad\text{for all }k\in\{0,\ldots,n\},
\label{eq:k_stability_rate}
\end{equation}
where $C_T$ may depend on the algorithm's hyperparameters but not on $n$.
Then, for $\beta=\frac{c}{\sqrt n}$ with any constant $c>0$, on the
confidence event of Theorem~\ref{th:th_gen_bound_beta_smooth},
\begin{equation}
\E_S[R(\cA_S)-\hat R_S(\cA_S)]
\le e^{\frac{c}{\sqrt n}}
    \frac{C_T\cdot K(c,\delta)}{c\cdot e\cdot\sqrt n}
=O\!\left(\frac{C_T}{\sqrt n}\right).
\label{eq:asymptotic-bound}
\end{equation}
If Eq.~\ref{eq:k_stability_rate} holds for every observed sample with
the same $C_T$, this conclusion holds with probability at least
$1-\delta$.
\end{proposition}

The controlling radius is approximately $\sqrt n/c$: this is
where the linear growth of neighborhood stability balances its
exponential discount in $B_\beta^*$. Choosing
$\beta=c/\sqrt n$ therefore aligns this radius with the scale
of the concentration term. Importantly, this proposition does \textit{not} consider the potential growth due to computational approximation; however, it does provide critical insight into the nature of our analysis relative to existing theory.

Proposition~\ref{prop:asymptotic} allows us to study our bound under
the assumptions of classical stability analyses. In particular, existing analyses often assume $\epsilon(S^\star)\le\frac{C_T}{n}$ for every
dataset $S^\star$, which directly implies
Eq.~\ref{eq:k_stability_rate}. Thus, we do not have to consider the growth in stability which implies the exponential discount is at most one,
$B_\beta^*(S_{\mathrm{obs}})\le\frac{C_T}{n}$. Substituting into
Theorem~\ref{th:th_gen_bound_beta_smooth} gives
\begin{equation}
\E_S[R(\cA_S)-\hat R_S(\cA_S)]
\le K(c,\delta)\frac{C_T}{n}
=O\!\left(\frac{C_T}{n}\right)
\label{eq:classical-rate-recovery}
\end{equation}
on the high probability event. Thus, for fixed $c$ and $\delta$, the
smooth-stability analysis preserves the classical expected-generalization
bound up to a multiplicative constant independent of $n$.

For regularized empirical risk minimization, the uniform-stability results of
\citet{bousquet2002stability} establish the required bound over all
datasets, with $C_T=O(\frac{L^2}{\lambda})$ under their regularity
assumptions, absorbing the kernel bound into $L$. Their
removal-to-replacement conversion changes only the constant. Since
uniform stability bounds each of our dataset-specific quantities,
Eq.~\ref{eq:classical-rate-recovery} recovers the corresponding
$O(\frac{L^2}{\lambda n})$ expected-generalization rate.

For SGD on smooth convex losses, we can obtain the same conclusion
using dataset-specific stability under the Lipschitz,
smoothness, and step-size assumptions of \citet{hardt2016train}. Their nonexpansive-update
argument bounds the loss change from replacing training example $i$
by $2L^2\sum_{t:\,i_t=i}\alpha_t$. Averaging these bounds in $\tilde\epsilon(S)$
(let this be $\bar\epsilon(S)$ or $\hat\epsilon(S)$) counts each update exactly once, giving
$\tilde\epsilon(S)\le\frac{2L^2}{n}\sum_{t=1}^T\alpha_t$
for every dataset $S$. Consequently, our analysis recovers
$O(\frac{L^2}{n}\sum_{t=1}^T\alpha_t)$ while retaining the fixed
algorithmic randomness used throughout this work.

While adopting the assumptions of prior work allows us to study
asymptotic behavior, our framework also permits sound computation
\textit{beyond} those assumptions, where the same rates are no longer
guaranteed. Its purpose is to extend the use of stability analysis to
algorithms for which the required analytical constants are unavailable.
The computed bounds inherit a vanishing rate when their certificates
are sufficiently tight, but their soundness does not require the
verification relaxation to satisfy the growth condition in
Proposition~\ref{prop:asymptotic}.

\subsection{A Parameter-Space Approach to Data-Specific Guarantees}
In this section, we focus on the computational aspects of our approach,
demonstrating how dataset-specific stability can be bounded through
reachability analysis. We begin with the reachable set in parameter
space \citep{sosnin2025abstract}.

\begin{definition}[Valid Parameter Space Domain]
\label{def:valid_parameter_space_domain}
\label{def:reachable_parameter_space_domain}
For an observed dataset $S$ and a learning algorithm $\cA$, a set
$\mathcal V_k(S)\subseteq\cH$ is a valid parameter space domain at radius
$k\in\{0,\ldots,n\}$ if it contains all hypotheses resulting from at
most $k$ substitutions in $S$. Formally,
\begin{equation}
\mathcal R_k(S):=\{\cA_{S^\star}:S^\star\in\cZ^n,\ d(S,S^\star)\le k\}
\subseteq\mathcal V_k(S),
\label{eq:valid-domain}
\end{equation}
where $\mathcal R_k(S)$ is the exact reachable parameter set. Since no two datasets
of size $n$ differ in more than $n$ coordinates, $\mathcal{R}_k(S) =
\mathcal{R}_n(S)$ for every $k \ge n$, and we take $\mathcal{V}_k(S) :=
\mathcal{V}_n(S)$ in that case.
\end{definition}

Practical computation of valid parameter domains can be done through
abstract interpretation \citep{wicker2024certification,sosnin2024certified}
or optimization over sound relaxations of training
\citep{sosnin2025abstract}. Recent work also provides exact constructions
for particular linear models \citep{sosnin2026exact} and shows how
relaxations of learning algorithms can inform algorithmic decisions
\citep{tsay2026relaxation}. For readers less familiar with this line of
work, we provide an introduction to Abstract Gradient Training (AGT) in
Appendix~\ref{app:agt}. Two containment lemmas used in the proofs
appear in Appendix~\ref{app:containment}.

\begin{assumption}[Effective Boundedness]
\label{ass:eff_boundedness}
There exist finite certified constants $\hat L\le\hat U$ such that
$\hat L\le\ell(h,z)\le\hat U$ for every $h\in\mathcal V_n(S)$ and
(i) every $z\in\operatorname{supp}(\cD)$ for on-average certification,
or (ii) every $z\in\cZ$ for uniform certification.
Write $M:=\hat U-\hat L$.
\end{assumption}

A radius-$n$ domain contains the output of $\cA$ (with fixed randomness) on every training
dataset. Thus, we can soundly cap the
resulting loss-difference bounds for an instance $z$ (for on-average) or at every instance (for uniform) at $M$. This does not require domains
computed independently at different radii to be nested.

\begin{definition}[Certified Loss Diameter]
\label{def:loss-diameter}
Let $\mathcal V_k(S)$ be a valid parameter space domain and $z\in\cZ$.
If $[\underline\ell_k(z),\overline\ell_k(z)]$ contains $\ell(h,z)$ for
all $h\in\mathcal V_k(S)$, then
$\varepsilon_k(z):=\overline\ell_k(z)-\underline\ell_k(z)$ is a
certified loss diameter at $z$. Where
Assumption~\ref{ass:eff_boundedness} applies, we cap this width at $M$.
The dependence on the observed sample $S$ is suppressed.
\end{definition}

\begin{proposition}[Loss Stability from Parameter Domains]
\label{prop:loss_stab_from_param_domains}
Let $\mathcal V_k(S)$ be valid and let $\varepsilon_k(z)$ be obtained as
in Definition~\ref{def:loss-diameter}. For any $S^\star,S^\dagger$ with
$d(S,S^\star)\le k$ and $d(S,S^\dagger)\le k$,
$|\ell(\cA_{S^\star},z)-\ell(\cA_{S^\dagger},z)|\le\varepsilon_k(z)$.
In particular, taking $S^\dagger=S$ bounds the loss change between the
observed model and any model reachable within radius $k$.
\end{proposition}

As with parameter bounds, there is a rich literature on propagating
instances through networks whose weights lie in abstract domains
\citep{wicker2022robust, sosnin2025abstract,wicker2024certification}.

\begin{theorem}[Valid Parameter Bounds Yield Stability Bounds]
\label{th:valid_param_bounds_yield_stab_bounds}
\label{cor:valset_estimate}
Let the domains be valid and let $\varepsilon_k(z)$ satisfy
Definition~\ref{def:loss-diameter} under
Assumption~\ref{ass:eff_boundedness}. For $0\le k\le n$, the following choices bound $\epsilon(S^\star)$
for every $S^\star$ with $d(S,S^\star)\le k$:
\begin{enumerate}[leftmargin=*,itemsep=0.5pt,topsep=0.5pt,parsep=0pt]
    \item \textbf{Population or uniform certification.} For
    $\epsilon=\hat\epsilon$, take
    $U_k(S)=\E_{z\sim\cD}[\varepsilon_{k+1}(z)]$; for
    $\epsilon=\bar\epsilon$, take
    $U_k(S)=\sup_{z\in\cZ}\varepsilon_{k+1}(z)$.
    \item \textbf{Validation-set certification.} For
    $\epsilon=\hat\epsilon$, let $Z_v=(z_1',\ldots,z_{n_v}')\sim\cD^{n_v}$
    be independent of $S$, with the domains fixed before observing $Z_v$.
    Assume $n_v\ge2$, and let
    $\hat\mu_r:=\frac{1}{n_v}\sum_{j=1}^{n_v}\varepsilon_r(z_j')$ and
    $\hat\sigma_r^2:=\frac{1}{n_v-1}\sum_{j=1}^{n_v}(\varepsilon_r(z_j')-\hat\mu_r)^2$
    be the sample mean and variance of the certified loss differences.
    With probability at least $1-\eta$, simultaneously for all $k$, take
    \begin{equation}
    U_k(S):=\min\!\left\{M,\
        \hat\mu_{k+1}
        +\sqrt{\frac{2\hat\sigma_{k+1}^2\log\frac{2n}{\eta}}{n_v}}
        +\frac{7M\log\frac{2n}{\eta}}{3(n_v-1)}\right\}.
    \label{eq:validation-certificate}
    \end{equation}
\end{enumerate}
Consequently, $B_\beta^*(S)\le\max_{0\le k\le n}U_k(S)e^{-\beta k}$.
Substitution into Theorem~\ref{th:th_gen_bound_beta_smooth} gives an
expected generalization bound with probability at least $1-\delta$ for
population and uniform certification, and
$1-\delta-\eta$ for validation-set certification.
\end{theorem}

\begin{wrapfigure}{r}{0.5\textwidth}
\begin{minipage}{0.5\textwidth}
\begin{algorithm}[H]
\small
\caption{Efficient upper bound for $B_\beta^*$}
\label{alg:smooth_comp}
\begin{algorithmic}[1]
\STATE \textbf{Input:} $S$, radii $\mathcal K\subseteq\{1,\ldots,n\}$, $\beta>0$.
\STATE \textbf{Initialize:} $G\in\mathbb R^{n+1}$ with $G_k=M$.
\FOR{$r\in\mathcal K$}
    \STATE Compute $\mathcal V_r(S)$ via AGT.
    \STATE Certify the required losses via IBP.
    \FOR{$k\in\{0,\ldots,n\}$ with $k+1\le r$}
        \STATE $G_k\gets\min\{G_k,U_k^{(r)}\}$.
    \ENDFOR
\ENDFOR
\STATE \textbf{Result:} $B_\beta^*(S)\le\max_k G_ke^{-\beta k}$.
\end{algorithmic}
\end{algorithm}
\end{minipage}
\end{wrapfigure}

\looseness=-1
The above theorem demonstrates how bounds on individual loss values
induce bounds on the quantities required for generalization. Validation-set certification uses an independent
sample and adds an empirical Bernstein concentration term
\citep{maurer2009empirical}. Naively, maximizing over $k$ appears to require a valid parameter domain
for every radius and hence linearly many additional training runs. In the
following section, we show that a sound bound can instead be obtained
with any constant number of parameter-domain computations.

\subsection{Bounding $\beta$-Smooth Sensitivity}
\label{sec:algorithm}

We now provide an efficient way to bound smooth stability. Recall that
AGT yields a valid domain,
$\mathcal R_r(S)\subseteq\mathcal V_r(S):=\mathrm{AGT}_r(S)$.
The central idea of Algorithm~\ref{alg:smooth_comp} is that a domain
computed at radius $r$ also contains every hypothesis reachable at a
smaller radius. Its loss certificates can therefore improve several
entries of the stability bound at once.

We first obtain the global bound $M$ by certifying the radius-$n$ domain
and its loss range. This supplies the initial value $G_k=M$ at every
index. For each chosen radius $r\in\mathcal K$, we compute the loss
certificates over $\mathcal V_r(S)$, intersected with the global domain
as above. The quantity $U_k^{(r)}$ is the bound in
Theorem~\ref{th:valid_param_bounds_yield_stab_bounds} with those
radius-$r$ certificates substituted for the radius-$(k+1)$ certificates.

For validation-set certification, the estimates must hold simultaneously
at the chosen radii. Allocating failure probability
$\frac{\eta}{|\mathcal K|}$ to each estimate replaces $n$ by
$|\mathcal K|$ inside the logarithm in
Eq.~\ref{eq:validation-certificate}. The collection $\mathcal K$ may
depend on $S$ but is fixed before examining validation data. If it is
empty, the algorithm simply returns the global bound $M$.

\begin{proposition}[Correctness of Algorithm~\ref{alg:smooth_comp}]
\label{prop:algorithm_correctness}
On the validity event of its certificates,
Algorithm~\ref{alg:smooth_comp} returns an upper bound on $B_\beta^*(S)$
for any $\mathcal K$ and any $\beta>0$. It requires at most
$|\mathcal K\cup\{n\}|$ parameter-domain runs, including the global
run used to obtain $M$. Thus any constant number of runs suffices for
soundness; the choice of radii affects the tightness of the bound.
\end{proposition}

Even a single global run gives a sound bound, while additional local
runs can reduce the certificate near the observed sample. The remaining
loss certification and postprocessing costs depend on the sample sizes.
As with other applications of AGT, over-approximation of the reachable
parameters can lead to conservative bounds; an unbounded global
certificate gives a vacuous bound. The default
$\beta=\frac{1}{\sqrt n}$ follows from Proposition~\ref{prop:asymptotic},
and Appendix~\ref{app:beta} discusses optimization of $\beta$ and the
corresponding confidence accounting.

\begin{figure}[t]
    \centering
    \includegraphics[width=\linewidth]{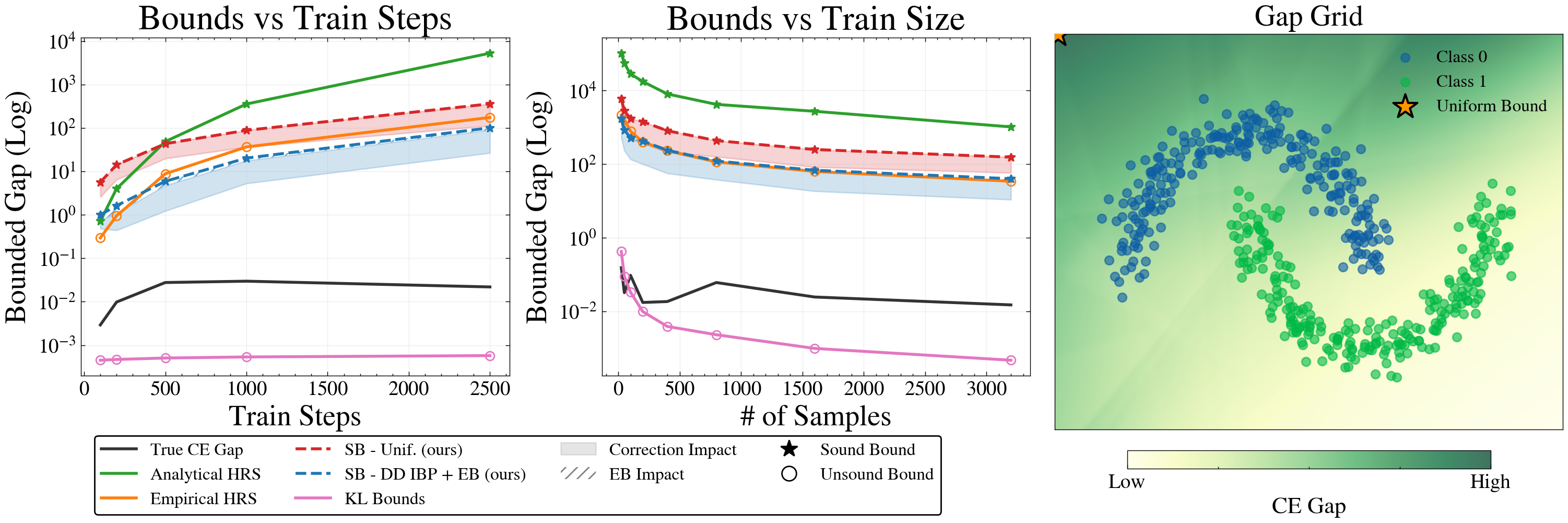}
    \caption{Bound gaps observed on the halfmoons dataset. The leftmost and center plots explore our bounds against benchmarks across number of train steps and train size. The rightmost figure plots the cross entropy loss gap alongside a halfmoons sample and a star for the uniform bound. Shaded region displays correction factor impact.}
    \label{fig:planar_halfmoons}
\end{figure}

\textbf{Limitations.}
\looseness=-1
While our bounds relax analytical assumptions required by classical stability analyses, their reliance on AGT introduces practical limitations. Importantly, AGT maintains soundness through over-approximation, i.e., computing a superset of the true reachable parameter set. As model depth or training time increases, the certified parameter-space domains can become overly conservative, potentially leading to vacuous bounds. This effect can be mitigated in part by larger batch sizes, as observed by \citet{sosnin2024certified}, and can even be eliminated by more computationally intensive variants of AGT that yield exact parameter-space bounds \citep{sosnin2026exact}. Moreover, our bounds will reflect future advances in abstract domain propagation. 

\looseness=-1
A related concern is that optimizing our certificate is guaranteed only to improve the upper bound on the generalization gap. Improvements may therefore reflect reduced AGT approximation artifacts rather than genuine improvements in generalization. Thus, they should be interpreted as certificates affected by both the learning algorithm and the verifier; sound always, but only as tight as the verifier allows. In \citet{tsay2026relaxation}, the authors demonstrate how to make relaxation-informed decisions about learning algorithms that can side-step this limitation.

\looseness=-1
Finally, although our method requires only a constant number of AGT runs, it remains more computationally expensive than closed-form analytical bounds (full AGT complexity can be found in \citet{wicker2024certification, sosnin2026exact}); we report empirical wall-clock times in Appendix~\ref{app:Experimental_Details} and note they are constant multiples of the standard train time.

\section{Experiments}\label{sec:Experiments}

We now validate our generalization framework experimentally on a set of datasets and models ranging from synthetic datasets (halfmoons) to more contemporary LLM classification head fine-tuning.
For further experimental details, please refer to Appendix~\ref{app:Experimental_Details}.

\begin{figure}[t]
    \centering
    \includegraphics[width=\linewidth]{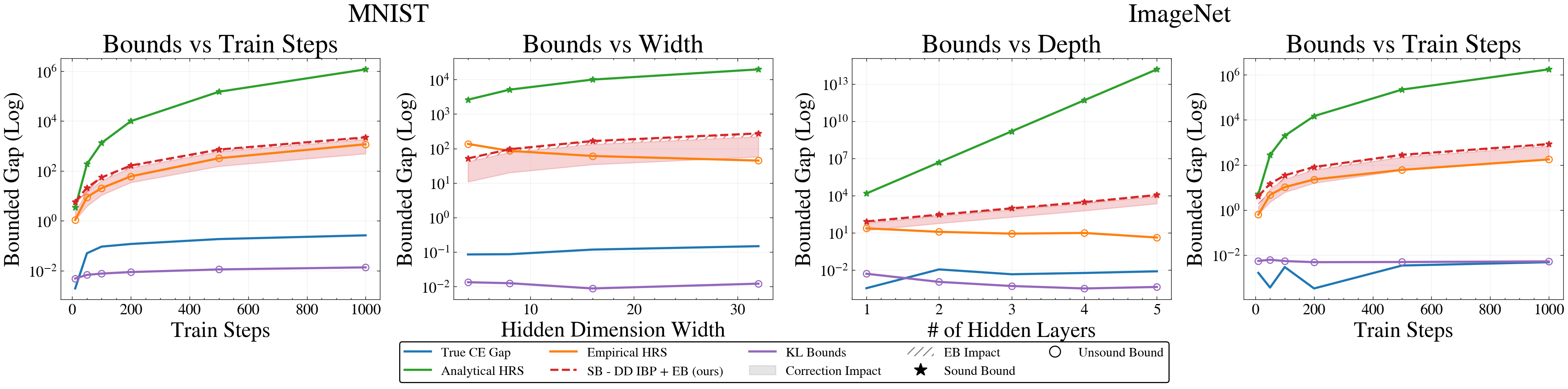}
    \caption{From left to right the plots show (i) the behavior of our data-dependent bounds across train steps on MNIST, (ii) bounds across different hidden layer widths, (iii) bound tightness for varying numbers of hidden layers on ImageNet, (iv) and the same train step plot as before but on ImageNet.}
    \label{fig:images}
\end{figure}

\textbf{Synthetic Datasets.} We begin by analyzing the performance of our bounds on the synthetic halfmoons dataset \citep{Pedregosaetal2011}. We compare our method to uniform bounds as per \citet{hardt2016train} in numerical and analytical settings; using the initials of the authors, we call these bounds HRS bounds. Additionally, we plot the data dependent bounds following \citet{kuzborskij2018data} and the true cross-entropy (CE) gap. In Figure~\ref{fig:planar_halfmoons}, we see from left to right that our bounds are consistently tighter than sound analytical bounds and competitive with unsound numerical bounds as we vary the numbers of train steps and train sizes compared to a suite of baselines. The shaded regions under our bounds represent the difference between $B_\beta^*(S)$ and the required correction in Theorem~\ref{th:th_gen_bound_beta_smooth}. We note that the \textit{KL Bound} (as per \citet{kuzborskij2018data}) consistently violates soundness, essentially just modeling the noise around the error. These findings further support the need for our formally sound bounds. Lastly, we point out that in the rightmost subfigure of Figure~\ref{fig:planar_halfmoons} we clearly demonstrate the main shortcoming of uniform bounds as the test instance with the globally worst-case loss stability (represented by a yellow star) is usually incurred far away from where data is actually observed.

\textbf{Image Classification.} To evaluate our bounds on a more difficult task we investigate the performance of our bounds on MNIST \citep{lecun1998gradient} and ImageNet \citep{russakovsky2015imagenetlargescalevisual}, problems where we train a small neural network and a classification head on top of a feature extractor for MNIST and a ResNet-18 \citep{he2015deepresiduallearningimage} for ImageNet.

\begin{figure}[t]
    \centering
    \includegraphics[width=\linewidth]{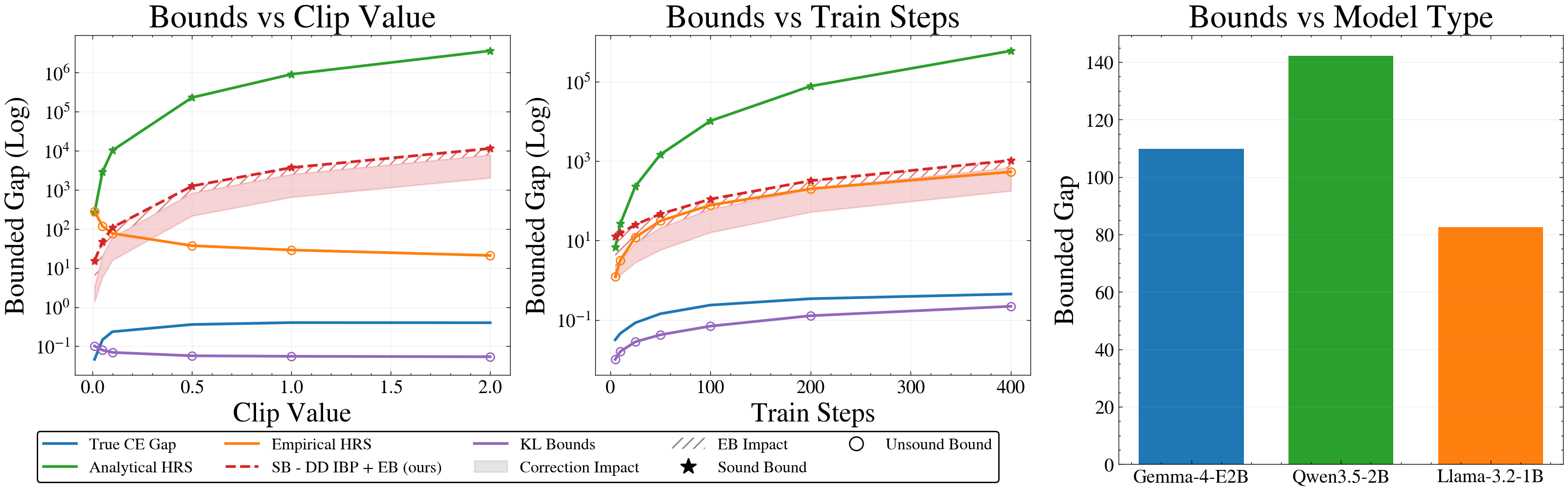}
    \caption{The bound behavior compared to baselines in LLM classification head fine-tuning. The left and center subfigures use Gemma-4-E2B, whereas the rightmost subfigure compares different foundation models.}
    \label{fig:llm}
\end{figure}
\looseness=-1
In Figure~\ref{fig:images}, we show the behavior of our data-dependent, on-average bound on MNIST in terms of train steps and model width, compared to various baselines similar to the planar experiments. Additionally, the two figures on the right show the bound performance in terms of model depth and number of train steps on ImageNet. We note that our bounds become looser with the depth of the underlying model, which is an inherent characteristic of interval bound propagation \citep{sosnin2025abstract}.

\textbf{LLM Fine-tuning.} We investigate the performance of our bounds when fine-tuning a classification head on top of a large language model on the Twitter hate speech dataset \citep{davidson2017automatedhatespeechdetection}. We evaluate our bounds on a model trained on the embeddings of different LLMs (Gemma-4-E2B \citep{gemma4e2b2026}, Qwen-3.5-2B \citep{qwen3.5, yang2025qwen3technicalreport}, Llama-3.2-1B \citep{llama3, llama32_1b}) in the binary classification task of hate/non-hate. Figure~\ref{fig:llm} once again shows that our bounds scale to more complex settings, maintaining soundness and reasonable tightness. In this suite of experiments, we note the behavior under changing clip values (e.g., gamma values for gradient clipping during training) and observe that smaller clip values yield tighter bounds.

\subsection{Ablations}

In Figure \ref{fig:ablations} on the very left, we see that with increasing the number of $k$-sensitivity values evaluated (with each $k$ being a random sample), that our bounds tighten as expected. The middle figure studies the asymptotic behavior of our bounds as $n\rightarrow \infty$, which we expect to decrease at a rate of $O(1/\sqrt{n})$ where $n$ is the number of training samples, recovering asymptotic patterns discussed in Proposition \ref{prop:asymptotic}. Lastly, we observe optimal $\beta$ values for increasing the sample sizes in the rightmost subfigure where we see that as sample size increases, the optimal $\beta$ decreases.

\begin{figure}[H]
    \centering
    \includegraphics[width=\linewidth]{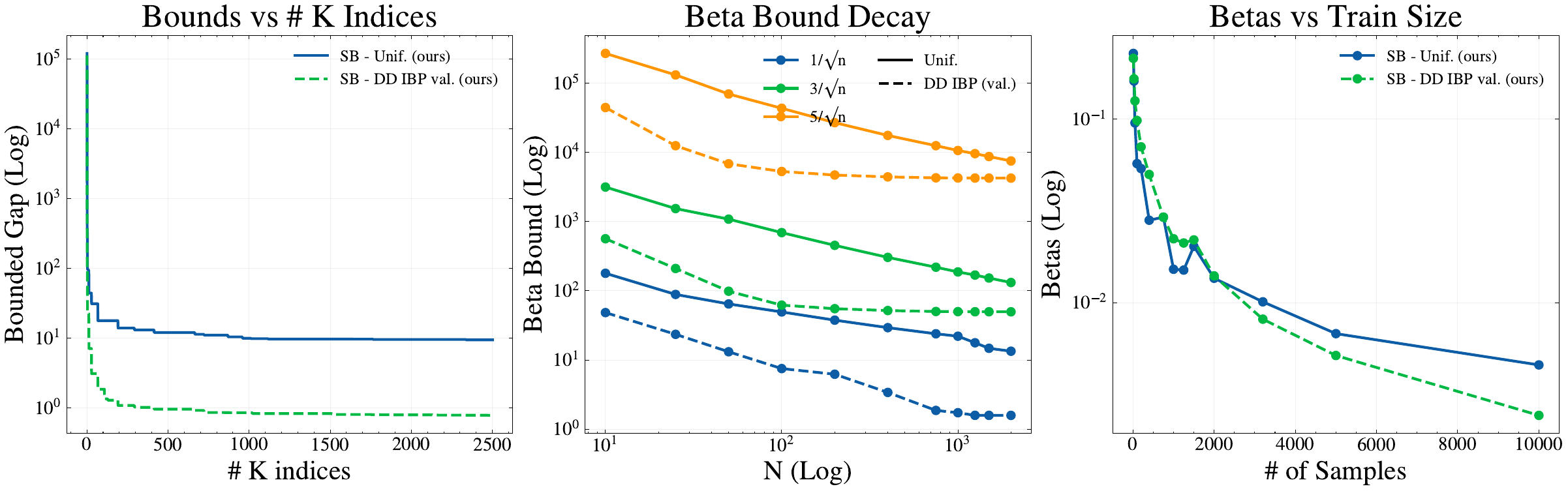}
    \caption{Small ablations on different properties of interest. Left to right: (i) investigation of bound tightness for increasing numbers of k evaluations in our practical algorithm, (ii) asymptotic bound behavior as $N\rightarrow\infty$ where $N$ is the sample size (excl. correction factor), and (iii) optimal $\beta$ values for increasing sample sizes.}
    \label{fig:ablations}
\end{figure}

\section{Conclusion}\label{sec:Conclusion}

We introduce a certification-enhanced framework for deriving sound, sample-specific generalization bounds from algorithmic stability. By recasting stability as a formal specification and using reachability analysis to certify the effect of dataset perturbations on the training trajectory, our approach extends classical stability-based guarantees beyond the restrictive analytical regimes in which they are typically applicable. The accompanying concentration analysis lifts these sample-specific certificates to high-probability distributional guarantees, while our practical algorithm computes the resulting bounds using only a constant number of certification runs. Across synthetic, image-classification, and language-model classification head fine-tuning tasks, the resulting bounds are consistently sound and substantially tighter than existing sound alternatives. These results suggest that formal certification can serve not only as a tool for post-hoc verification, but also as a principled mechanism for producing rigorous and practically informative generalization guarantees for modern learning algorithms.

While we have instantiated our framework for three classical notions of stability, the fundamental idea of using verification techniques to bound the generalization gap is not restricted to algorithmic stability. Natural extensions of this line of work could consider deriving similar bounds for PAC-Bayes or information theoretic bounds. Additionally, certificates for the realized gap of the observed sample, rather than the expected gap, may be of greater practical relevance; the bound in Appendix~\ref{app:realized-gap} is a first step in this direction.

\begin{ack}
We acknowledge the support of the UKRI Prosperity Partnership Scheme (FAIR) under EPSRC
Grant EP/V056883/1.
\end{ack}

\bibliography{iclr2027_conference}

\appendix
\clearpage
\section{Proofs}\label{app:proofs}

\subsection{Proof of Proposition~\ref{prop:jensen}}
\begin{proof}
Let $z_i'\sim\cD$ be independent of $S$, and let $S^{(i)}$ denote
replacement of $z_i$ by $z_i'$. Exchanging $z_i$ and $z_i'$ gives
\begin{align}
\E_S[R(\cA_S)-\hat R_S(\cA_S)]
&=\frac{1}{n}\sum_{i=1}^n\E_{S,z_i'}
  [\ell(\cA_S,z_i')-\ell(\cA_{S^{(i)}},z_i')]
  \label{eq:ghost-fresh}\\
&=\frac{1}{n}\sum_{i=1}^n\E_{S,z_i'}
  [\ell(\cA_{S^{(i)}},z_i)-\ell(\cA_S,z_i)].
  \label{eq:ghost-training}
\end{align}
For the first identity, $\ell(\cA_{S^{(i)}},z_i')$ has the same
expectation as $\ell(\cA_S,z_i)$. For the second, $S^{(i)}$ has law
$\cD^n$ and is independent of $z_i$, so its expected loss at $z_i$ is
the expected population risk.

The average in Eq.~\ref{eq:ghost-fresh} is at most its largest summand,
$\hat\epsilon(\cD)$. Bounding each signed loss change by its absolute
value, and then taking the supremum over the replacement and test
instance, also bounds the expected gap by $\bar\epsilon(\cD)$.

For the on-average case, replace the particular replacement point by a
supremum and apply Jensen's inequality to the maximum over $i$:
\begin{align*}
\hat\epsilon(\cD)
&\le\max_i\E_S\left[\E_z\sup_{z_i'}
  [\ell(\cA_S,z)-\ell(\cA_{S^{(i)}},z)]\right]\\
&\le\E_S\left[\max_i\E_z\sup_{z_i'}
  [\ell(\cA_S,z)-\ell(\cA_{S^{(i)}},z)]\right]
 =\E_S[\hat\epsilon(S)].
\end{align*}
Here $z\sim\cD$ is independent of $S$, whereas $z_i'$ in the supremum
ranges over all admissible replacements. For the uniform case,
Jensen's inequality for the supremum gives
\begin{align*}
\bar\epsilon(\cD)
&=\sup_{i,z_i',z}\E_S
  |\ell(\cA_S,z)-\ell(\cA_{S^{(i)}},z)|\\
&\le\E_S\left[\sup_{i,z_i',z}
  |\ell(\cA_S,z)-\ell(\cA_{S^{(i)}},z)|\right]
 =\E_S[\bar\epsilon(S)].
\end{align*}
Finally, taking the supremum over $z_i'$ in each summand of
Eq.~\ref{eq:ghost-training} gives
$\E_S[R(\cA_S)-\hat R_S(\cA_S)]\le\E_S[\bar\epsilon(S)]$.
Combining these observations proves the proposition. The inequalities
also hold as extended upper bounds when a stability quantity is infinite
and the expected gap exists.
\end{proof}

\subsection{Proof of Proposition~\ref{prop:properties_of_B}}
\begin{proof}
The upper-bound property follows by taking $S^\star=S$ in
Definition~\ref{def:beta_smooth_stability_at_S}. For the second property,
the triangle inequality gives
$d(S',S^\star)\le d(S',S)+d(S,S^\star)$, and hence
\[
\epsilon(S^\star)e^{-\beta d(S,S^\star)}
\le e^{\beta d(S,S')}\epsilon(S^\star)e^{-\beta d(S',S^\star)}.
\]
Taking the supremum over $S^\star$ gives
$B_\beta^*(S)\le e^{\beta d(S,S')}B_\beta^*(S')$.
Interchanging $S$ and $S'$ and taking logarithms therefore yields
$|\log B_\beta^*(S)-\log B_\beta^*(S')|\le\beta d(S,S')$
whenever the envelope is positive and finite.

We record the boundary cases needed in the concentration proof.
Since the Hamming diameter is $n$, the same inequalities show that
positivity and finiteness at one dataset imply a uniform positive lower
bound and a finite upper bound everywhere. If $B_\beta^*(S)=0$ at one
dataset, every term defining it is non-negative and has a strictly
positive weight, so $\epsilon(S^\star)=0$ for every $S^\star$.
Thus the envelope is identically zero in this case.
\end{proof}

\subsection{Proof of Theorem~\ref{th:th_gen_bound_beta_smooth}}
\begin{proof}
Recalling Propositions~\ref{prop:jensen} and \ref{prop:properties_of_B},
we have
\begin{equation}
\E_S[R(\cA_S)-\hat R_S(\cA_S)]
\le\E_S[\epsilon(S)]\le\E_S[B_\beta^*(S)].
\label{eq:mean-envelope-chain}
\end{equation}
If $B_\beta^*$ vanishes, this proves the stated upper bound of zero.
Otherwise let $g(S):=\log B_\beta^*(S)$. As established in the preceding
proof, $g$ is bounded and changing one coordinate of $S$ changes $g(S)$
by at most $\beta$.

McDiarmid's inequality bounds the lower tail:
$\Pr(g(S)\le\E[g(S)]-t)\le\exp(-\frac{2t^2}{n\beta^2})$.
Setting the right-hand side to $\delta$ shows that, with probability at
least $1-\delta$,
\begin{equation}
\E[g(S)]\le g(S_{\mathrm{obs}})
              +\beta\sqrt{\frac{n}{2}\log\frac{1}{\delta}}.
\label{eq:log-envelope-tail}
\end{equation}
The corresponding exponential-moment bound is
$\E[e^{g(S)-\E[g(S)]}]\le e^{\frac{n\beta^2}{8}}$.
To see this directly, reveal the independent coordinates of $S$ in
order and form the Doob martingale for $g$. Each increment has
conditional mean zero and conditional range of length at most $\beta$.
Hoeffding's lemma bounds its conditional exponential moment by
$e^{\frac{\beta^2}{8}}$; iterating conditional expectations proves the
claim.

Combining the moment bound with Eq.~\ref{eq:log-envelope-tail} gives
\begin{align*}
\E_S[B_\beta^*(S)]
&=e^{\E[g(S)]}\E[e^{g(S)-\E[g(S)]}]\\
&\le B_\beta^*(S_{\mathrm{obs}})
 \exp\!\left(\beta\sqrt{\frac{n}{2}\log\frac{1}{\delta}}
                    +\frac{n\beta^2}{8}\right).
\end{align*}
Substituting into Eq.~\ref{eq:mean-envelope-chain} completes the proof.
Replacing $B_\beta^*(S_{\mathrm{obs}})$ by a sound numerical upper
bound only enlarges the right-hand side; the numerical bound need not
itself be smooth.
\end{proof}

\subsection{Proof of Proposition~\ref{prop:asymptotic}}
\begin{proof}
For an observed sample satisfying Eq.~\ref{eq:k_stability_rate}, group
the datasets in Definition~\ref{def:beta_smooth_stability_at_S} by their
Hamming distance from $S_{\mathrm{obs}}$. This gives
\begin{align*}
B_\beta^*(S_{\mathrm{obs}})
&\le\max_{0\le k\le n}\frac{C_T(k+1)}{n}e^{-\beta k}\\
&\le\frac{C_Te^\beta}{n}\sup_{u\ge0}u e^{-\beta u}
 =\frac{C_Te^\beta}{e\beta n}.
\end{align*}
The function $u e^{-\beta u}$ is maximized at $u=\frac{1}{\beta}$.
We enlarged the optimization domain to all $u\ge0$, so the inequality
holds for every $\beta>0$. At $\beta=\frac{c}{\sqrt n}$, it becomes
$B_\beta^*(S_{\mathrm{obs}})\le
 e^{\frac{c}{\sqrt n}}\frac{C_T}{ce\sqrt n}$.
The prefactor in Theorem~\ref{th:th_gen_bound_beta_smooth} reduces to
$K(c,\delta)$, which gives Eq.~\ref{eq:asymptotic-bound} on that
confidence event. Both $K(c,\delta)$ and $e^{\frac{c}{\sqrt n}}\le e^c$
are bounded independently of $n$ for fixed $c,\delta$.
\end{proof}

\subsection{Containment Properties of Parameter Domains}
\label{app:containment}
\begin{lemma}[Neighbor Containment of Parameter Domains]
\label{lem:neighbor_containment}
If $d(S,S')\le1$, then
$\mathcal R_k(S')\subseteq\mathcal R_{\min\{k+1,n\}}(S)$
for $0\le k\le n$.
\end{lemma}
\begin{proof}
Let $\cA_{S^\star}\in\mathcal R_k(S')$. The triangle inequality gives
$d(S,S^\star)\le d(S,S')+d(S',S^\star)\le k+1$.
Since also $d(S,S^\star)\le n$, the hypothesis belongs to the set on
the right.
\end{proof}

\begin{lemma}[Monotonicity of Valid Parameter Domains in $k$]
\label{lem:monotonicity_of_reachable_domains}
If $\mathcal V_r(S)$ is valid at radius $r$, it is also valid at every
radius $j\le r$.
\end{lemma}
\begin{proof}
Every $S^\star$ with $d(S,S^\star)\le j$ also satisfies
$d(S,S^\star)\le r$. Consequently,
$\mathcal R_j(S)\subseteq\mathcal R_r(S)\subseteq\mathcal V_r(S)$.
This shows that a computed domain can be reused at a smaller radius;
it does not require independently computed domains to be nested.
\end{proof}

\subsection{Proof of Proposition~\ref{prop:loss_stab_from_param_domains}}
\begin{proof}
Fix $z$, and let $S^\star,S^\dagger$ be within radius $k$ of $S$.
By validity of the parameter domain,
$\cA_{S^\star},\cA_{S^\dagger}\in\mathcal R_k(S)\subseteq\mathcal V_k(S)$.
Both losses therefore lie in
$[\underline\ell_k(z),\overline\ell_k(z)]$, so their absolute difference
is at most the width of this interval. After intersection with the
global domain, Assumption~\ref{ass:eff_boundedness} also bounds the
difference by $M$ on the relevant instance domain. Taking the smaller
of the two bounds gives $\varepsilon_k(z)$ and proves the claim.
The special case $S^\dagger=S$ follows because $d(S,S)=0$.
\end{proof}

\subsection{Proof of Theorem~\ref{th:valid_param_bounds_yield_stab_bounds}}
\begin{proof}
Fix $S$ and $S^\star$ with $d(S,S^\star)\le k$. Let
$(S^\star)^{(i)}$ denote replacement of its $i$th coordinate by
$z_i'\in\cZ$. The triangle inequality, with the Hamming diameter $n$,
gives
\begin{equation}
\cA_{S^\star},\ \cA_{(S^\star)^{(i)}}
\in\mathcal R_{k+1}(S)\subseteq\mathcal V_{k+1}(S).
\label{eq:both-in-domain}
\end{equation}
The inclusion is preserved after intersection with the global domain.

\emph{Population and uniform certification.}
For each relevant $z$, Proposition~\ref{prop:loss_stab_from_param_domains}
gives
$|\ell(\cA_{S^\star},z)-\ell(\cA_{(S^\star)^{(i)}},z)|
 \le\varepsilon_{k+1}(z)$.
Taking the supremum over replacements, averaging over $z$, and maximizing
over $i$ proves
$\hat\epsilon(S^\star)\le\E_z[\varepsilon_{k+1}(z)]$.
Taking suprema over $i,z_i',z$ gives the uniform case.

\emph{Validation-set certification.}
Condition on $S$ and its parameter domains. For each radius $r$, the
function $\varepsilon_r(z)$ is fixed and takes values in $[0,M]$ on the
support of $\cD$. The validation points are i.i.d.\ and independent of $S$, so the
empirical Bernstein inequality \citep[Theorem~4]{maurer2009empirical},
applied to the $[0,1]$-valued variables $\varepsilon_r(z_j')/M$ and
rescaled by $M$, therefore gives
\[
\E_z[\varepsilon_r(z)]
\le\hat\mu_r
   +\sqrt{\frac{2\hat\sigma_r^2\log\frac{2}{\alpha}}{n_v}}
   +\frac{7M\log\frac{2}{\alpha}}{3(n_v-1)}
\]
with probability at least $1-\alpha$, where $\hat\mu_r$ and
$\hat\sigma_r^2$ are the sample mean and unbiased sample variance from
Theorem~\ref{th:valid_param_bounds_yield_stab_bounds}. Set $\alpha=\frac{\eta}{n}$ and
take a union bound over $r=1,\ldots,n$. Since the population stability
is also bounded by $M$, this proves the capped estimate in
Eq.~\ref{eq:validation-certificate} for all required radii.

If $M=0$, all these statements hold
deterministically.

\emph{Smooth stability and confidence.}
On the validity event of the selected certificates, every $S^\star$
satisfies $\epsilon(S^\star)\le U_{d(S,S^\star)}(S)$. Therefore
\[
B_\beta^*(S)
=\sup_{S^\star}\epsilon(S^\star)e^{-\beta d(S,S^\star)}
\le\max_{0\le k\le n}U_k(S)e^{-\beta k}.
\]
Substituting into Theorem~\ref{th:th_gen_bound_beta_smooth} gives the
claimed confidence. In validation mode, the two events fail with
probability at most $\delta$ and $\eta$, so their intersection has
probability at least $1-\delta-\eta$. Independence between these events
is not required. The validation bounds hold conditionally for every
$S$, and thus also jointly over training and validation draws.
\end{proof}

\subsection{Proof of Proposition~\ref{prop:algorithm_correctness}}
\begin{proof}
For the proof, let
$a_k(S):=\sup_{S^\star:\,d(S,S^\star)\le k}\epsilon(S^\star)$.
We show the invariant $G_k\ge a_k(S)$ for each $0\le k\le n$ on the
certificate-validity event. Initially $G_k=M$, which is valid by
effective boundedness.

Consider a chosen radius $r$ and any index $k$ with $k+1\le r$.
Lemma~\ref{lem:monotonicity_of_reachable_domains} permits use of
$\mathcal V_r(S)$ wherever the proof of
Theorem~\ref{th:valid_param_bounds_yield_stab_bounds} used
$\mathcal V_{k+1}(S)$. Consequently, the corresponding
$U_k^{(r)}$ still bounds $a_k(S)$. Taking the minimum
of two valid upper bounds preserves the invariant, while unvisited
indices retain $M$.

At termination,
$B_\beta^*(S)\le\max_{0\le k\le n}a_k(S)e^{-\beta k}
 \le\max_{0\le k\le n}G_ke^{-\beta k}$.
The argument uses containment of exact reachable sets, without assuming
that independently computed certificates are monotone. There is one
global run at radius $n$ and one run at each additional chosen radius,
so the total is at most $|\mathcal K\cup\{n\}|$, reusing the global
domain when $n\in\mathcal K$.

For validation-set certification, the estimates are made simultaneous
using failure probability $\frac{\eta}{|\mathcal K|}$ per chosen radius.
Combining this event with the fixed-$\beta$ concentration event gives
the confidence in Theorem~\ref{th:valid_param_bounds_yield_stab_bounds}.
If $\mathcal K$ is empty, the deterministic fallback $M$ suffices.
\end{proof}

\section{Bounding the Realized Gap}
\label{app:realized-gap}

The bounds in the main text concern the expected generalization gap.
The global certificate used to obtain $M$ also yields a high-probability
bound on the realized gap of the observed sample, though it is
considerably looser.

\begin{corollary}[Realized-gap certificate]\label{cor:realized-gap}
Suppose Assumption~\ref{ass:eff_boundedness} holds and $\cA$ is deterministic. From the single
$k=n$ certification run, define
\[
  \hat\gamma := \sup_{z\in\cZ}\bigl(\overline\ell_n(z)-\underline\ell_n(z)\bigr),
  \qquad
  \widehat M := \sup_{z\in\cZ}\overline\ell_n(z)-\inf_{z\in\cZ}\underline\ell_n(z),
\]
the certified loss diameter and range over $\mathcal V_n(S)$. Since both outputs of any
neighboring pair lie in $\mathcal V_n(S)$, these quantities are admissible under any
distribution $\cD$ over $\cZ$. Let $S\sim\cD^n$ and $\delta\in(0,1)$. Then the following holds with probability at least $1-\delta$ over the draw of
the training sample; in particular, it applies to the observed sample
$S_{\mathrm{obs}}$.
  \begin{equation}\label{eq:realised_gap_bounds}
    R(\cA_{S_{\mathrm{obs}}})-\hat R_{S_{\mathrm{obs}}}(\cA_{S_{\mathrm{obs}}})\;\le\; B_{\mathrm{McD}}
    := \hat\gamma\Bigl(1+\sqrt{2n\ln\tfrac{1}{\delta}}\Bigr)
    +\widehat M\sqrt{\frac{\ln(1/\delta)}{2n}}.
  \end{equation}
\begin{proof}[Proof sketch.]
The gap $F(S) := R(\cA_S) - \hat{R}_S(\cA_S)$ has bounded differences $2\gamma^\star + M^\star/n$, where $\gamma^\star$ and $M^\star$ are the deterministic minimal analogues of $\hat{\gamma}$ and $\widehat{M}$ over $\mathcal{Z}^n$ (finite by Assumption~\ref{ass:eff_boundedness}(ii)). Indeed, the population term, $\big|R(\cA_S) - R(\cA_{S^{(i)}})\big|$, moves by at most $\gamma^\star$, and the empirical term, $\big|\hat R_S(\cA_S) - \hat R_{S^{(i)}}(\cA_{S^{(i)}})\big|$,  at most $\gamma^\star + M^\star/n$, because its replaced summand compares losses at different evaluation points and is therefore bounded only by the range. With $\mathbb{E}_S[F] \le \gamma^\star$ by exchangeability, McDiarmid's inequality gives Eq.~\ref{eq:realised_gap_bounds} with $(\gamma^\star,M^\star)$ in place of $(\hat\gamma,\widehat M)$. Both $\hat{\gamma} \ge \gamma^\star$ and $\widehat{M} \ge M^\star$ hold surely by AGT soundness, so the certified constants substitute on the same event.
\end{proof}
\end{corollary}

The sharper concentration bounds for uniformly stable algorithms of
\citet{bousquet2020sharper} could be used to tighten
Eq.~\ref{eq:realised_gap_bounds}; we leave this to future work.

\section{A Brief Introduction to Abstract Gradient Training}
\label{app:agt}

This appendix provides a self-contained introduction to Abstract Gradient
Training (AGT)~\citep{sosnin2024certified, sosnin2025abstract} for readers
unfamiliar with reachability analysis in parameter space. Our goal is to
convey enough intuition for the reader to understand how the valid parameter
space domains $\mathcal{V}_k(S)$ used in Section~\ref{sec:data-specific-bounds} are
constructed, and how the per-instance loss bounds $\varepsilon_k(z)$ in
Proposition~\ref{prop:loss_stab_from_param_domains} are obtained from them. We refer the
reader to~\citet{sosnin2025abstract} for a comprehensive treatment.

\subsection{Background: Interval Bound Propagation}
\label{app:agt:ibp}

The fundamental primitive underlying AGT is \emph{interval bound propagation}
(IBP)~\citep{gowal2019effectivenessintervalboundpropagation}, a form of abstract interpretation that
tracks element-wise lower and upper bounds on quantities of interest as they
flow through a computation. Given an input represented by an interval
$[\underline{x}, \overline{x}]$ and an affine map $y = W x + b$, IBP computes
sound bounds on the output by separating the positive and negative parts of
$W$:
\[
    \underline{y} = W^{+} \underline{x} + W^{-} \overline{x} + b,
    \qquad
    \overline{y} = W^{+} \overline{x} + W^{-} \underline{x} + b,
\]
where $W^{+} = \max(W, 0)$ and $W^{-} = \min(W, 0)$ are taken element-wise.
For a monotonic activation $\sigma$ (e.g., ReLU), bounds propagate
element-wise: $\underline{\sigma(y)} = \sigma(\underline{y})$ and
$\overline{\sigma(y)} = \sigma(\overline{y})$. Iterating this procedure
through a network yields sound, if generally loose, bounds on every
intermediate activation and on the final output. Tighter relaxations such as
CROWN~\citep{zhang2018efficientneuralnetworkrobustness} replace the activation step with linear
upper and lower envelopes; we use IBP throughout for clarity, but the
framework is agnostic to the choice of relaxation.

\subsection{From Input-Space to Parameter-Space Reachability}
\label{app:agt:param_space}

Standard verification of neural networks~\citep{gehr2018ai2, katz2017reluplex}
treats the network's \emph{weights} as fixed and reasons about reachable
\emph{outputs} given an interval over inputs. AGT inverts this perspective.
For a fixed observed dataset $S$ and a learning algorithm $\mathcal{A}$, AGT
treats the weights themselves as the abstract object and reasons about the
set of all parameter trajectories that the algorithm could produce when run
on any dataset $S^{\star}$ within Hamming distance $k$ of $S$.

Concretely, let $\theta_t$ denote the network parameters at optimization step
$t$, and let $g_t = \nabla_\theta \mathcal{L}(\theta_t; B_t)$ be the gradient
on a (possibly mini-)batch $B_t \subseteq S$. AGT maintains, at every step
$t$, an interval $[\underline{\theta}_t, \overline{\theta}_t]$ that
soundly contains $\theta_t^{S^{\star}}$ for every $S^{\star}$ within
$k$-Hamming distance of $S$:
\[
    \forall\, S^{\star} \in \mathcal{Z}^n
    \text{ with } d(S, S^{\star}) \leq k,
    \quad
    \theta_t^{S^{\star}} \in
    \big[\underline{\theta}_t,\, \overline{\theta}_t\big].
\]
This containment is the parameter-space analogue of soundness in standard
verification. The set
\[
    \mathcal V_k(S) := \mathrm{AGT}_k(S)
    = \big[\underline{\theta}_T,\, \overline{\theta}_T\big]
\]
returned at the final step $T$ is the valid parameter space domain $\mathcal{V}_k(S)$
of Definition~\ref{def:valid_parameter_space_domain}.

\subsection{From Parameter Bounds to Loss Bounds}
\label{app:agt:loss_bounds}

The parameter-space domain $\mathcal V_k(S)$ is, on its own, a guarantee about
\emph{weights}; what we ultimately need are guarantees about \emph{losses}
on test instances $z \in \mathcal{Z}$ (cf.\ Proposition~\ref{prop:loss_stab_from_param_domains}).
This is exactly the setting standard verification was designed for: with
$\mathcal V_k(S)$ playing the role of an abstract weight set, we propagate a fixed
test input $z$ forward through the network using IBP, treating the weights
as intervals at each layer. This yields sound bounds $[\underline{\ell}_k(z), \overline{\ell}_k(z)]$ on
$\ell(h,z)$ for every $h \in \mathcal{V}_k(S)$, realizing
Definition~\ref{def:loss-diameter} with
$\varepsilon_k(z) = \min\{M,\overline{\ell}_k(z) - \underline{\ell}_k(z)\}$
after the global-range intersection described in the main text.
Since both $\ell(h,z)$ and $\ell(h',z)$ lie in this interval for any
$h, h' \in \mathcal{V}_k(S)$, we obtain
$|\ell(h,z) - \ell(h',z)| \le \varepsilon_k(z)$, and a fortiori
$|\ell(\mathcal{A}_{S^\star}, z) - \ell(\mathcal{A}_{S^\dagger}, z)|
 \le \varepsilon_k(z)$ for any $S^\star, S^\dagger$ within Hamming
distance $k$ of $S$.

\subsection{Properties of AGT}
\label{app:agt:what_we_get}

Three properties of AGT are essential. \emph{Soundness}: $\mathcal V_k(S)$ provably
contains $\mathcal{A}_{S^{\star}}$ for every $S^{\star}$ within Hamming
distance $k$ of $S$, so any quantity bounded over $\mathcal V_k(S)$ is bounded over
the relevant set of trained models without exception. \emph{Algorithm
specificity}: because the bounds are propagated through the actual
optimization trajectory, $\mathcal V_k(S)$ reflects the structure of the specific
algorithm and dataset rather than worst-case analytical reasoning over
arbitrary algorithms. \emph{Tractability}: AGT propagates bounds through the training steps,
and the procedure scales to
the model sizes considered in our experiments, including the LLM
classification heads of Section~\ref{sec:Experiments}. The empirical
wall-clock overhead is reported in Table~\ref{tab:wall_clock_time}.

\section{Experimental Details}\label{app:Experimental_Details}

Here we report some experimental details and hyperparameters used. The code is to be made public upon publication of this work. All experiments were run on two NVIDIA L40 GPUs. In addition we report wall-clock times for k computation in Table~\ref{tab:wall_clock_time} and additional experiment-level wall-clock times associated with Figures~\ref{fig:llm_type_avg}, \ref{fig:llm_clip_avg}, and \ref{fig:img_model_depth_avg}.
\subsection{Planar}

All halfmoons experiments use a simple MLP for classification purposes with the following hyperparameters for train step and train size sweeps found in Table~\ref{tab:hparams-halfmoons-sweeps}.

We note that arrays of length greater than one indicate that this dimension was varied across test runs for a given figure. Additionally, we point out that \texttt{n\_cells\_per\_dim} corresponds to the number of cells per dimension used for exhausting the input space for our uniformly stable bounds (i.e., \texttt{n\_cells\_per\_dim} = 100 implies $100 \times 100$ cells to be evaluated).

For the final subfigure, we evaluate $5000 \times 5000$ cells.

\subsection{Image Classification}

The two tasks used in image classification were odd/even classification for MNIST and animal/non-animal classification for ImageNet. The animal classes used were Chihuahua, Golden Retriever, Tabby, Goldfinch, and Bald Eagle, whereas the non-animal samples consisted of Beach wagon, Sports car, Minivan, Airliner, and Tank. Left to right, the hyperparameters used can be found in Tables~\ref{tab:hparams-mnist} and \ref{tab:hparams-imagenet}.

\subsection{LLM Fine-tuning}

LLM fine-tuning consisted of fine-tuning a classification head on the embeddings produced by different open-source foundation models. The specific task at hand was hate/non-hate speech classification. The first two experiments in Figure~\ref{fig:llm} use Gemma-4-E2B as the underlying model. The hyperparameters can be found in Tables~\ref{tab:hparams-llm-sweeps} and \ref{tab:hparams-llm-models}.

\subsection{Ablations}\label{app:ablations}

All ablations in Figure \ref{fig:ablations} are run on the halfmoons dataset, the relevant hyperparameters corresponding to Figure~\ref{fig:ablations} are in Table~\ref{tab:hparams-ablations-runs}.

\begin{table}[t]
  \centering
    \caption{Hyperparameter configurations for $\beta$ ablations in Figure \ref{fig:ablations}. Swept values in \textbf{bold}.}
  \begin{tabular}{lll}
    \toprule
    Hyperparameter & \textsc{train\_size} sweep & all-$k$ run \\
                   & (panels 2--3)              & (panel 1) \\
    \midrule
    Noise $\sigma$            & 0.3   & 0.3 \\
    Train size $m$            & \textbf{10, 25, 50, 100, 200, 400, 750, 1000,} & 2500 \\
                              & \textbf{1250, 1500, 2000, 3200, 5000, 10000}   &      \\
    Validation size           & 1000  & 1000 \\
    Training steps            & 200   & 200 \\
    Learning rate             & 0.05  & 0.05 \\
    Clipping bound $\gamma$   & 0.1   & 0.1 \\
    Hidden width              & 16    & 16 \\
    Confidence $\delta$ (smooth bound) & 0.05 & 0.05 \\
    Confidence $\delta_{\mathrm{EB}}$ & 0.05 & 0.05 \\
    \bottomrule
  \end{tabular}
  \label{tab:hparams-ablations-runs}
\end{table}

\begin{table}[ht]
\centering
\caption{Wall-clock times in seconds for computing bounds based on 100 values of $k$, for a model trained on 500 sample points over 100 steps, standard deviation reported over 10 runs.}
\begin{tabular}{lcc}
\toprule
Setting & Mean [s] & Std [s] \\
\midrule
Halfmoons -- Uniform & 3.4878 & 0.1041 \\
Halfmoons -- DD & 2.5993 & 0.0904 \\
MNIST -- DD & 2.5743 & 0.1908 \\
ImageNet -- DD & 2.5233 & 0.1217 \\
LLM -- DD & 1.9881 & 0.0298 \\
\bottomrule
\end{tabular}
\label{tab:wall_clock_time}
\end{table}

\begin{table}[t]
  \centering
  \caption{Halfmoons sweeps over train size and train steps. Swept values are in
  \textbf{bold}.}
  \label{tab:hparams-halfmoons-sweeps}
  \begin{tabular}{lll}
    \toprule
    Hyperparameter & \textsc{train\_size} sweep & \textsc{train\_steps} sweep \\
    \midrule
    Noise $\sigma$              & 0.3   & 0.3 \\
    Train size $m$              & \shortstack[l]{\textbf{25, 50, 100, 200,}\\\textbf{400, 800, 1600, 3200}} & 1600 \\
    Validation size             & 1000  & 1000 \\
    Training steps (full batch) & 1000  & \shortstack[l]{\textbf{10, 50, 100, 200,}\\\textbf{500, 1000, 2500}} \\
    Learning rate               & 0.1   & 0.05 \\
    Clipping bound $\gamma$     & 0.1   & 0.1 \\
    Hidden width                & 16    & 16 \\
    Confidence $\delta$ (smooth bound) & 0.05 & 0.05 \\
    Confidence $\delta_{\mathrm{EB}}$  & 0.05 & 0.05 \\
    \bottomrule
  \end{tabular}

\end{table}

\begin{table}[t]
  \centering
  \caption{MNIST sweeps over train steps and hidden layer dimension. The certified model is an MLP head over MNIST pixels projected onto
  their leading 50 principal components. Swept values are in \textbf{bold}.}
  \label{tab:hparams-mnist}
  \begin{tabular}{lll}
    \toprule
    Hyperparameter & \textsc{train\_steps} sweep & \textsc{width} sweep \\
    \midrule
    Train size $m$           & 500   & 500 \\
    Validation size          & 1000  & 1000 \\
    Training steps           & \textbf{10, 50, 100, 200, 500, 1000} & 200 \\
    Learning rate            & 0.05  & 0.05 \\
    Clipping bound $\gamma$  & 0.1   & 0.1 \\
    Hidden width             & 16    & \textbf{4, 8, 16, 32} \\
    Hidden layers            & 1     & 1 \\
    Confidence $\delta$ (smooth bound) & 0.05 & 0.05 \\
    Confidence $\delta_{\mathrm{EB}}$  & 0.05 & 0.05 \\
    \bottomrule
  \end{tabular}
\end{table}

\begin{table}[t]
  \centering
  \caption{ImageNet sweeps over model depth and train steps. The certified model is an MLP head over frozen ResNet-18 features. Swept values are in \textbf{bold}.}
  \label{tab:hparams-imagenet}
  \begin{tabular}{lll}
    \toprule
    Hyperparameter & \textsc{depth} sweep & \textsc{train\_steps} sweep \\
    \midrule
    Train size $m$     & 500   & 500 \\
    Validation size    & 1000  & 1000 \\
    Training steps     & 200   & \textbf{10, 50, 100, 200, 500, 1000} \\
    Learning rate      & 0.05  & 0.05 \\
    Clipping bound $\gamma$ & 0.1 & 0.1 \\
    Hidden width       & 16    & 16 \\
    Hidden layers      & \textbf{1, 2, 3, 4, 5} & 1 \\
    Confidence $\delta$ (smooth bound) & 0.05 & 0.05 \\
    Confidence $\delta_{\mathrm{EB}}$  & 0.05 & 0.05 \\
    \bottomrule
  \end{tabular}
\end{table}

\begin{table}[t]
  \centering
  \caption{Hyperparameters for gamma value and train step sweeps. The certified model is a small MLP head over frozen LLM embeddings. Swept values are in \textbf{bold}.}
  \label{tab:hparams-llm-sweeps}
  \begin{tabular}{lll}
    \toprule
    Hyperparameter & \textsc{clip} sweep & \textsc{train\_steps} sweep \\
    \midrule
    Train size $m$     & 500   & 500 \\
    Validation size    & 500   & 500 \\
    Test size          & 500   & 500 \\
    Training steps     & 100   & \textbf{5, 10, 25, 50, 100, 200, 400} \\
    Learning rate      & 0.05  & 0.05 \\
    Clipping bound $\gamma$ & \textbf{0.01, 0.05, 0.1, 0.5, 1.0, 2.0} & 0.1 \\
    Hidden width       & 16    & 16 \\
    Hidden layers      & 1     & 1 \\
    Confidence $\delta$ (smooth bound) & 0.05 & 0.05 \\
    Confidence $\delta_{\mathrm{EB}}$  & 0.05 & 0.05 \\
    \bottomrule
  \end{tabular}
\end{table}

\begin{table}[t]
  \centering
  \caption{Hyperparameter configuration for embedding-model comparison.}
  \label{tab:hparams-llm-models}
  \begin{tabular}{ll}
    \toprule
    Hyperparameter & Value \\
    \midrule
    Train / val / test size & 500 / 500 / 500 \\
    Training steps   & 100 \\
    Learning rate    & 0.05 \\
    Clipping bound $\gamma$ & 0.1 \\
    Hidden width     & 16 \\
    Hidden layers    & 1 \\
    Confidence $\delta$ (smooth bound) & 0.05 \\
    Confidence $\delta_{\mathrm{EB}}$  & 0.05 \\
    \bottomrule
  \end{tabular}
\end{table}

\subsection{Selecting the Decay Parameter}\label{app:beta}
The decay parameter affects tightness but does not require additional
parameter-domain certification. Theorem~\ref{th:th_gen_bound_beta_smooth}
holds for every fixed $\beta>0$, chosen independently of the sample. The
schedule $\beta=\frac{c}{\sqrt n}$, and in particular $\beta=\frac{1}{\sqrt n}$, is a
search-free choice motivated by Proposition~\ref{prop:asymptotic}.

The implementation can optimize $\beta$ by inexpensive postprocessing of
the profile $G_0,\ldots,G_n$. For $a=\sqrt{\frac{n}{2}\log\frac{1}{\delta}}$, the
logarithm of the full computed generalization bound is
\begin{equation}
 \log\!\left[\max_{0\le k\le n}G_ke^{-\beta k}
                   e^{a\beta+\frac{n\beta^2}{8}}\right]
 =\max_{k:\,G_k>0}(\log G_k-\beta k)
                   +a\beta+\frac{n\beta^2}{8}.
 \label{eq:beta-objective}
\end{equation}
If every $G_k=0$, the bound is zero for every $\beta$. Otherwise,
\eqref{eq:beta-objective} is convex. On a segment where index $k$ is
active, its stationary point is
\[
 \beta_k=\frac{4}{n}(k-a).
\]
This is a candidate only when it is positive and lies in that index's
active segment. A minimizer can also occur at a kink between active
indices or at the boundary of a prescribed search interval; the
stationary-point formula alone is not a global solution.

A confidence guarantee for a choice made after observing $S$ requires
simultaneous validity. One simple protocol is to specify a finite grid
$\mathcal B$ of size $J$ in advance, apply the concentration theorem at
failure probability $\frac{\delta}{J}$ for each grid point, and select the
smallest resulting bound. A union bound gives failure probability at
most $\delta$. In validation mode, the same simultaneous validation
certificate event can be reused across the grid, giving total failure
probability at most $\delta+\eta$. The corrected concentration prefactor
must also be used in \eqref{eq:beta-objective}. Continuous line search
remains a numerical optimization of the displayed expression; its
statistical guarantee requires a corresponding simultaneous argument.

\subsection{Result Reproducibility}

In a subset of experiments, we show that our method is generally stable and results are reproducible by running them across 5 different seeds and then plotting averages and error bars showing the standard deviation. Please refer to Figures~\ref{fig:llm_type_avg}, \ref{fig:llm_clip_avg}, and \ref{fig:img_model_depth_avg} for further details.

\begin{figure}
    \centering
    \includegraphics[width=0.5\linewidth]{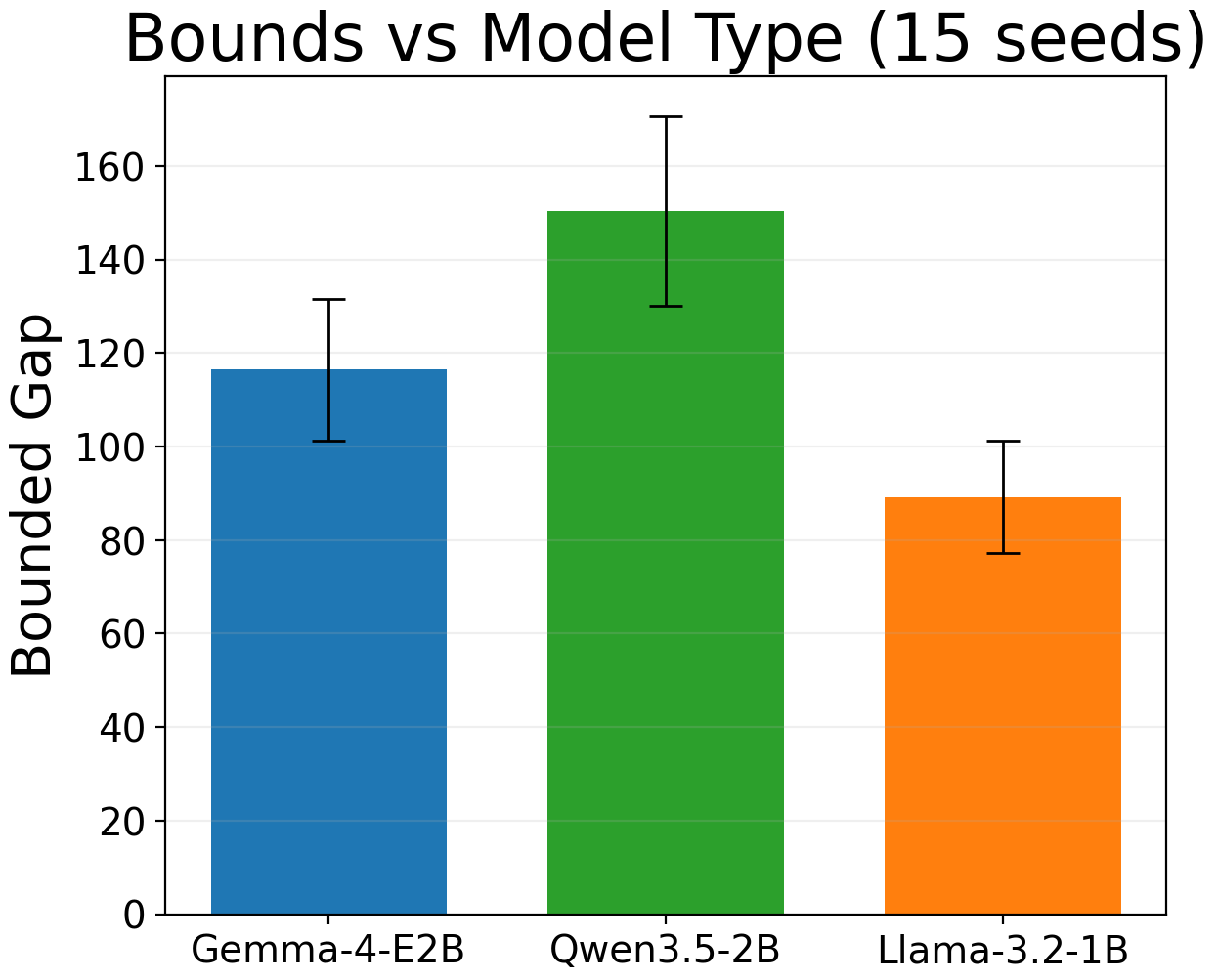}
    \caption{LLM model type experiment across 5 different seeds. Lines show the average with error bars representing one standard deviation. Run time across 15 runs [mean $\pm$ std]: $14.20 \pm 0.65$ s.}
    \label{fig:llm_type_avg}
\end{figure}

\begin{figure}
    \centering
    \includegraphics[width=0.5\linewidth]{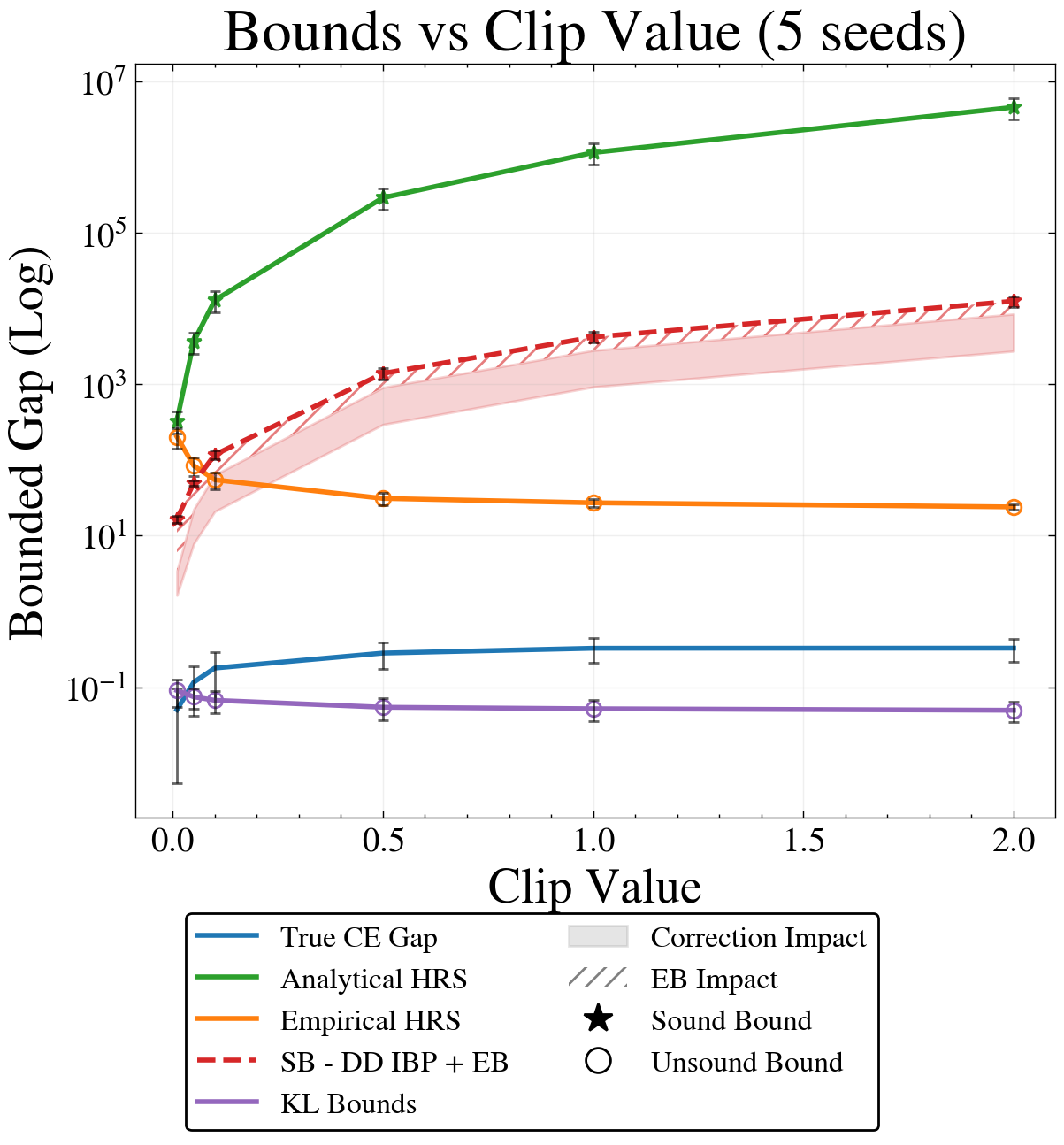}
    \caption{LLM clip value experiment across 5 different seeds. Lines show the average with error bars representing one standard deviation. Run time across 5 runs [mean $\pm$ std]: $29.02 \pm 0.82$ s.}
    \label{fig:llm_clip_avg}
\end{figure}

\begin{figure}
    \centering
    \includegraphics[width=0.5\linewidth]{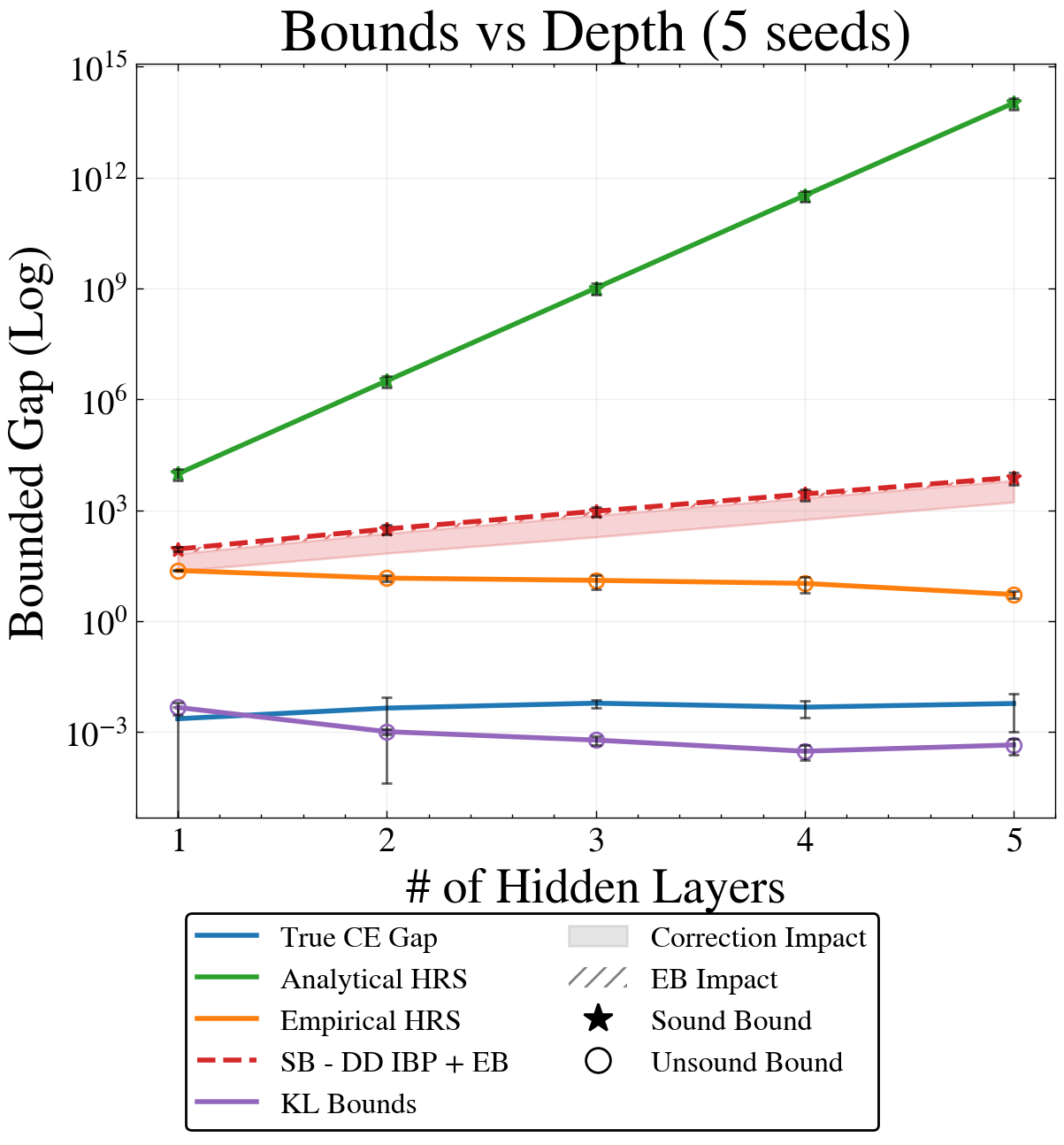}
    \caption{ImageNet model depth experiment across 5 different seeds. Lines show the average with error bars representing one standard deviation. Run time across 5 runs [mean $\pm$ std]: $48.85 \pm 1.01$ s.}
    \label{fig:img_model_depth_avg}
\end{figure}

\subsection{HRS Bound Failure Modes}

A concrete example of how our bounds are less restrictive in terms of required assumptions is when looking at the bounds vs the learning rate. In Figure~\ref{fig:learning_rate} we observe the point at which $\alpha<\frac{2}{\beta}$ for a $\beta$-smooth loss function no longer holds, technically rendering HRS bounds invalid for larger learning rates.

\begin{figure}
    \centering
    \includegraphics[width=0.5\linewidth]{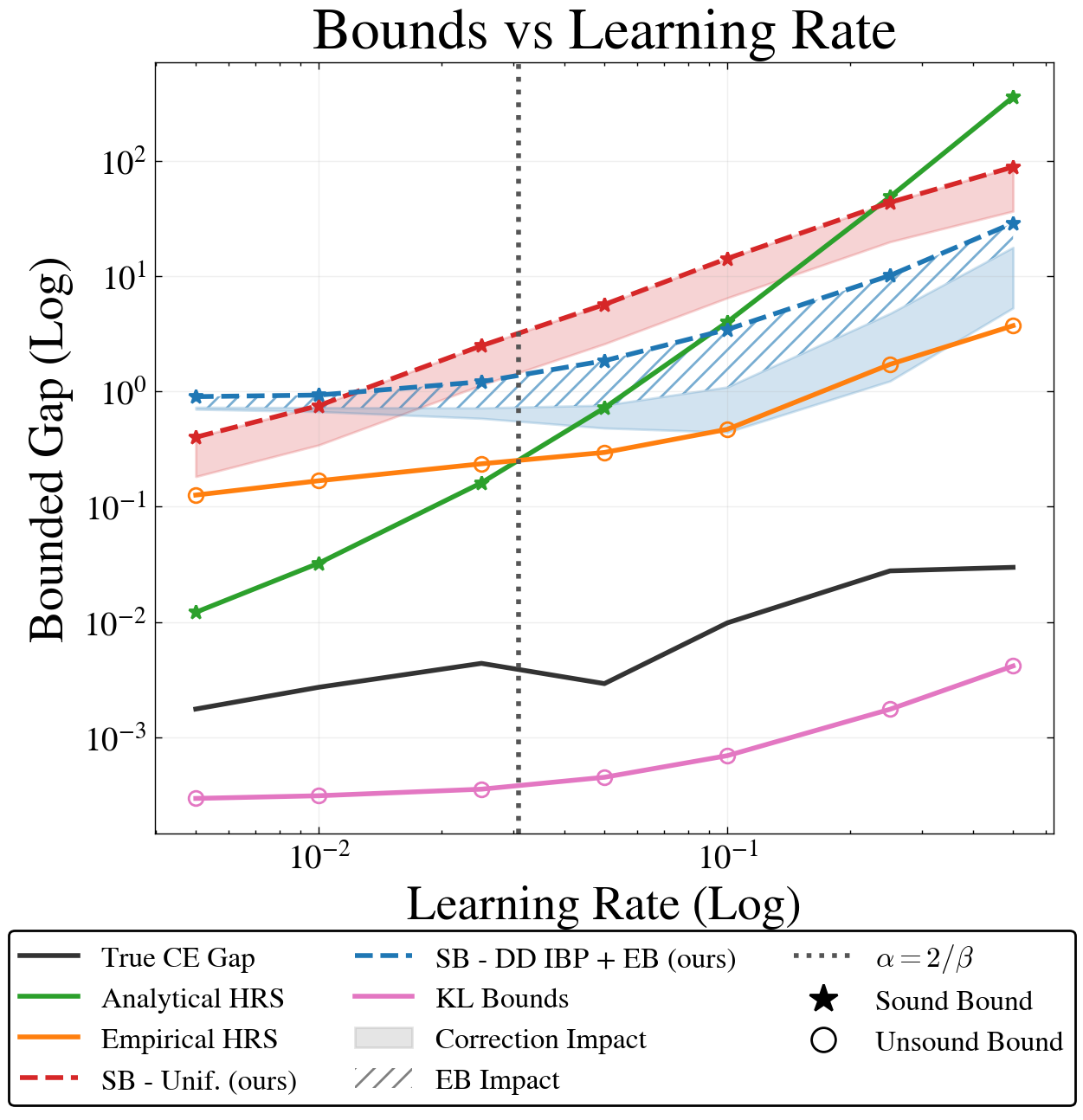}
    \caption{Evaluating generalization bounds over different learning rates. The vertical line shows the point where HRS bounds violate their requirement of $\alpha<\frac{2}{\beta}$ for a $\beta$-smooth loss function.}
    \label{fig:learning_rate}
\end{figure}

\section{Further Related Works}
\subsection{Assumption-Free Stability}

There exists a line of literature that is reasonably similar to our work in its goal. More specifically, \citet{soloff2024bagging} and its extension to ranking problems, \citet{liang2025assumptionfreestabilityrankingproblems}, also aim to bound the generalization error of an algorithm with a limited number of assumptions. They mainly achieve this using different bagging algorithms and placing some assumptions on the sampling distribution to generate individual bags as opposed to the learning algorithm itself.

However, it is important to note that \citet{soloff2024bagging} and \citet{liang2025assumptionfreestabilityrankingproblems} primarily focus on a slightly different notion of algorithmic stability, \textit{leave-one-out stability}, which compares a dataset with one of size $n-1$, whereas our setting compares datasets of the same size. \citet{soloff2024bagging} state their bound of interest as
\begin{equation}
    \E[\ell(\mathcal A_S,z) - \ell(\mathcal A_{S^{\backslash i}},z)]\leq\epsilon,
\end{equation}
where $S^{\backslash i}$ is the dataset $S$ \textbf{without} the sample at index $i$ as opposed to our quantity of interest which is based on $S^{(i)}$ where the sample at index $i$ is \textbf{replaced} with any data point.

While \citet{soloff2024bagging} obtain bounds with few assumptions on the learning algorithm, there is still a suite of assumptions placed on the sampling distribution $Q_n$ used to sample the individual ``bags''.

\end{document}